\documentclass[letterpaper]{article} % DO NOT CHANGE THIS
\usepackage[preprint]{aaai2027}  % DO NOT CHANGE THIS
\usepackage[hyphens]{url}  % DO NOT CHANGE THIS
\usepackage{graphicx} % DO NOT CHANGE THIS
\usepackage{natbib}  % DO NOT CHANGE THIS AND DO NOT ADD ANY OPTIONS TO IT
\usepackage{caption} % DO NOT CHANGE THIS AND DO NOT ADD ANY OPTIONS TO IT
\usepackage{algorithm}
\usepackage{algorithmic}
\usepackage{todonotes}
\usepackage{amsmath}
\usepackage{amssymb}
\usepackage{amsthm}
\usepackage{float}

\usepackage{xcolor}

\newtheorem{property}{Property}

\newtheorem{proposition}{Proposition}

\usepackage{newfloat}
\usepackage{listings}
\DeclareCaptionStyle{ruled}{labelfont=normalfont,labelsep=colon,strut=off} % DO NOT CHANGE THIS
\floatstyle{ruled}
\newfloat{listing}{tb}{lst}{}
\floatname{listing}{Listing}

\usepackage{booktabs}

\title{%In defence of 
A Theory of Post-hoc Debate Judgement}
\author{
    Xiang Yin\textsuperscript{\rm 1}\equalcontrib,
    Adam Dejl\textsuperscript{\rm 1}\equalcontrib,
    Antonio Rago\textsuperscript{\rm 2},
    Lihu Chen\textsuperscript{\rm 1},
    Francesca Toni\textsuperscript{\rm 1}
}
\affiliations{
    \textsuperscript{\rm 1}Imperial College London, UK\\
    \textsuperscript{\rm 2}King's College London, UK\\
    \{x.yin20, lihu.chen, f.toni\}@imperial.ac.uk, antonio.rago@kcl.ac.uk
}

\newcommand{\ft}[1]{\textcolor{black}{#1}} 
 
\newcommand{\ad}[1]{\textcolor{black}{#1}} 
\begin{document}

\maketitle

\begin{abstract}
Debates have recently emerged as a useful methodology for agentic AI to improve performance as well as to aid explainability% and user engagement
.
%, including \AR{via} contestability of \delete{predictions}\AR{models' outputs}
%For example, LLM-empowered agents may debate internally (with themselves) and/or externally (with other agents). 
In many settings where debates are used, debates' outcomes and resulting outputs are determined post-hoc by external \emph{judges}, often LLMs. In this paper we develop and test a novel theory of \emph{debate judgement} applicable to all settings where agents engage in debates by providing \emph{pros and cons} for their opinions therein. Specifically, we identify a number of formal properties that debate judgement may be required to satisfy in general, as concerns reproducibility, robustness, groundedness and explainability. Then, we explore their satisfaction %formally and/or 
experimentally, for claim verification settings, for two specific alternative debate judgement methods: variants of the \emph{LLMs as a judge} idea and formal \emph{semantics} drawn from computational argumentation. We show that the two methods give similar accuracy performances but the
%former may lack formal guarantees that the latter brings
\ft{latter shows better property satisfaction}. Overall, our study indicates argumentation semantics as an ideal candidate for principled judges in debate-driven AI.

%TO reDO ...... Debates have recently emerged as a useful  methodology in a number of deep learning settings, to improve performance as well as to aid explainability and user engagement, including \AR{via} contestability of \delete{predictions}\AR{models' outputs}. For example, LLM-empowered agents may debate internally and externally, with other agents, to improve their answers.\todo{check/add} In most settings where debates are used, their outcomes and  resulting predictions  are predominantly determined by external \emph{judges} (e.g. LLMs).  In this paper we argue that, for LLM-empowered agents\todo{whose opinions are individual predictions supported by pro-con explanations?}, if these judges are non-deterministic and lack formal guarantees, they cannot serve as faithful explanations for contestable\todo{AR: I think the contestable thing makes it weaker than focusing on accuracy also, but if you want to focus on it, maybe it should be in the title? WE DO NOT DO ACCURACY AT ALL, SO NOT SURE} predictions. In this context, we advocate the use of \emph{formal semantics} drawn from computational argumentation as principled judges for debate-driven explainable and contestable AI. 
\end{abstract}

\section{Introduction}

Debates have recently emerged as a useful methodology for agentic AI to improve performance~\citep{debate-safety2018,tillmann2025literature, debateLLMs-2024,li2024survey}. % as well as to aid explainability\todo{add a reference - is anybody making this point? we cannot use argllms - as sepcial case of debate - if none drop} and user engagement\todo{add ref - if none we should drop}.
When agents are empowered by LLMs, they may debate internally (with themselves), e.g. for explainable and contestable claim verification~\cite{freedman2025argumentative}, and/or externally (with other agents), e.g. for modelling collaboration mechanisms \cite{Zhang_24}. 
Debate frameworks may thus involve multiple agents to generate %diverse 
potentially different 
answers and %then 
argue with each other, which may lead to %refined
better answers~\citep{debate-safety2018,debateLLMs-2024,li2024survey}.
%\ft{While we give an abstract framework for debate judgement, to be applied broadly across various debate processes,} 
%existing work has mainly focused on designing the debate process itself.
One such line of research studies the structure of debating agents~\cite{tillmann2025literature},  ranging from \textit{homogeneous} settings, where all agents share the same base model, to \textit{heterogeneous} settings that adopt models with different capabilities or roles. 
Early work mainly uses homogeneous agents, e.g.~\cite{du2024improving}, which limits diversity and may lead to error aggregation, whereas more recent work introduces heterogeneous agents to improve debate quality, e.g.~\cite{chen2024reconcile}.
Another line investigates debate protocols with multiple rounds, which often improve performance over single-round debate~\cite{liang2024encouraging}. However, recent studies show that increasing rounds alone brings diminishing returns~\cite{humulti2025}, and debate performance is more influenced by diversity and argument confidence~\cite{zhu2026demystifying, lin-hooi-2025-enhancing}.
More broadly, debate has been applied to a wide range of tasks, including reasoning enhancement~\cite{du2024improving}, claim verification~\cite{freedman2025argumentative}, AI safety~\cite{debate-safety2024}, fairness~\cite{ki2025multiple}, 
and explainability~\cite{kim2024can}.

In many settings where debates are used, their outcomes and resulting outputs are determined post-hoc by external \emph{judges}.
These are often LLMs~\cite{liang2024encouraging, humulti2025, feng2025m}, suitably prompted with information drawn from the debates.
While these LLM-based judges have shown good performance in a number of tasks, they inherit issues from the underpinning LLM, e.g. their outputs cannot be explained or contested, as is the case with other uses of LLMs~\cite{freedman2025argumentative}. In general, 
an understanding of the formal properties of judges, which are orthogonal to
downstream performance, is lacking, potentially hindering the quality of debate judgement in high-stakes settings where performance alone is not sufficient.

In this paper we develop a novel theory of \emph{debate judgement} that is applicable to all settings where agents engage in debates by providing \emph{pros and cons} for their opinions therein. 
Specifically, we identify a number of formal properties that debate judgement may be required to satisfy in general, as concerns reproducibility (in terms of deterministic behaviour and agents' permutation independence), robustness (with respect to small variations of agents' opinions and judgement method), groundedness (as concerns faithfully drawing from agents' opinions)  and %explainability
contestability (with respect to the agents' ability to influence the judge). We see these properties not as strict requirements on all debate judges, but rather as a way to guide the choice of appropriate judges for debate settings of choice.

Finally, we explore the satisfaction of these properties %formally and/or 
experimentally, for claim verification settings, for two specific alternative debate judgement methods: variants of the \emph{LLMs as a judge} idea and formal \emph{semantics} drawn from computational argumentation (see ~\cite{argumentation17,Cyras_21} for an overview of this research area and its use to support explainability). 
For the experiments, we consider two debate scenarios:
one where LLM agents generate their opinions in isolation, following internal debate, and one where LLM agents take turns to refine their opinions.
For both scenarios,
judges may be presented the raw agents' opinions or an aggregation of these opinions, before being asked to decide on a final output. 
The experiments show that the two methods give similar accuracy performances in all scenarios, but 
%the former may lack formal guarantees that the latter brings
\ft{latter shows better property satisfaction}.

%moving this to conclusions
% Overall, we make the following contributions:

% \begin{itemize}
%     \item We define a novel set of formal properties for post-hoc debate judgement in the abstract.
%     \item We instantiate the resulting theory to the setting of claim verification, in a number of different scenarios depending on whether agents conduct debates internally or directly interact with one another.
%     \item We explore the satisfaction of the proposed properties formally and experimentally, when the judges are either LLMs or argumentation semantics. 
% \end{itemize}
% In summary,
% our study indicates argumentation semantics as an ideal candidate for principled judges in debate-driven AI.

\section{Related Work}

% \todo{adapted from bias paper at AAAI2026, to be massaged....NEEDS TO BE STRUCTURED TOO, WE NEED TO ADD WHO THE JUDGE IS IN EACH CASE...} 

% survey: \\
% 1. Literature Review Of Multi-Agent Debate For Problem-Solving\\
% 2. A Survey on LLM-based Multi-Agent Systems

Despite the progress seen in debate-based approaches, existing work mainly focuses on improving debate generation, while paying little attention to how %the resulting 
debates should be judged. 
%In this work, we treat the 
We see 
debate judgement as a principled object of study, and we thus focus here on the works in this area of research.

\noindent\textbf{Non-argumentative Debate Judgement}
%In many debate-based systems, the final decision is determined by post-hoc external judges. 
Existing approaches mainly adopt two forms of debate judgement.
The first form uses voting methods to aggregate outputs of debate agents, including majority voting~\cite{yin2023exchange,chan2024chateval} and weighted voting~\cite{wang2023discussion}. 
While simple and efficient, these methods consider only the final answers and largely ignore the debate itself. The second relies on LLMs as a judge, where one or more LLMs evaluate the debate and determine the final outcome~\cite{liang2024encouraging, humulti2025, feng2025m}.
Although LLM-based judges generally outperform voting-based methods, they remain sensitive to inconsistency and positional bias~\cite{wang2024large}. 
In contrast, we study post-hoc debate judgement by comparing LLMs as a judge with formal semantics drawn from computational argumentation, demonstrating that the former may lack formal guarantees that the latter brings.

\noindent\textbf{Argumentative Debate Judgement}
Given their dialectical nature, argumentation's many frameworks have long been used as a means for representing, reasoning about and judging debates. 
These debates concerning judgement were initially exclusively between humans, whether amongst the general public \cite{Lawrence_23}, online users \cite{Karamlou_19,Young_21,Sia_22,Ruiz-Dolz_23,Oluokun_24}, judgemental forecasters \cite{Irwin_22,Gorur_25} or politicians \cite{Goffredo_25}.
%Perhaps the most common task concerning the analysis of these debates is argument mining, in which arguments are extracted from debate text for representation, analysis or downstream tasks \cite{Lippi_16,Ruiz-Dolz_21,Goffredo_23,Roush_24}.
However, aligning with recent findings showing that debates between machines may improve their performance along different metrics \cite{du2024improving}, argumentative approaches have been introduced to judge debates between machines in various tasks, e.g. to determine classifications of images \cite{visualdebates24,Kori_25} and bias detection \cite{Ayoobi_26}, amongst others \cite{Rago_23,ayoobi2023sparx,Bezou-Vrakatseli_24}.
More recently, the framing of LLMs' reasoning as debates in the form of argumentation frameworks, where the judgement determines a claim's verification, has demonstrated benefits in explainability and contestability \cite{freedman2025argumentative}, with other improvements becoming possible when a multi-agent approach is taken \cite{Gorur_25}.
Further, \citet{Sanayei_25} use quantitative argumentation frameworks to assess whether LLMs are well-placed to judge debates. However, none of the above approaches introduce a general model for debate judgement. 

\section{Background%: Computational Argumentation
}
\label{sec:back}
We %give definition of formal 
use Quantitative Bipolar Argumentation Frameworks (QBAFs) \cite{baroni2015automatic}, i.e. tuples $%\mathcal{Q}=
\left\langle\mathcal{A}, \mathcal{R}^{-}, \mathcal{R}^{+}, \tau \right\rangle$, where $\mathcal{A}$ is a finite set of \emph{arguments};
$\mathcal{R}^{-} \subseteq \mathcal{A} \times \mathcal{A}$ is a binary \emph{attack} relation;
$\mathcal{R}^{+} \subseteq \mathcal{A} \times \mathcal{A}$ is a binary \emph{support} relation;
$\mathcal{R}^{-} \cap \mathcal{R}^{+} = \emptyset$;
$\tau: \mathcal{A} \rightarrow [0,1]$ is a \emph{base score function}.
The %base score function $\tau$ 
latter assigns an a-priori belief to arguments.

The structure of QBAFs is often shown graphically, as in Figure \ref{fig_qbaf}.
This QBAF illustrates a simple family debate on the claim ``We should go to the zoo this Sunday.'' Dad gives two con arguments, about needing rest after work (Con2) and possible crowds or long queues (Con3). Mum gives one con argument about rain and animals staying inside (Con1), and one pro argument about family time (Pro3). The child gives two pro arguments, based on Dad's promise (Pro1) and learning about wild animals (Pro2).

The \emph{dialectical strength} of arguments in QBAFs can be evaluated by \emph{(gradual) semantics} 
$\sigma: \mathcal{A} \rightarrow [0,1]$, 
as defined, for example, in \cite{baroni2015automatic,amgoud2018evaluation,Potyka18,potyka2021interpreting}.
In this paper, we focus on the DF-QuAD semantics \cite{rago2016discontinuity} due to its broad applicability~\cite{rago2016discontinuity,kotonya2019gradual,cocarascu2019extracting,chi2021optimized}.
In DF-QuAD, for any argument $A\in\mathcal{A}$, \(\sigma(A)=\tau(A)-\tau(A)\cdot(v_{Aa}-v_{As})\) if \(v_{Aa}\geq v_{As}\), and \(\sigma(A)=\tau(A)+(1-\tau(A))\cdot(v_{As}-v_{Aa})\) if \(v_{Aa}<v_{As}\),
where \(
v_{Aa}=1-\prod_{\left \{ X \in \mathcal{A} \mid (X,A) \in \mathcal{R^{-}} \right \} }(1-\sigma(X))
\) is the \emph{aggregation strength} of all the attackers against $A$, and \(
v_{As}=1-\prod_{\left \{ X \in \mathcal{A} \mid (X,A) \in \mathcal{R^{+}} \right \} }(1-\sigma(X))
\) is the \emph{aggregation strength} of all the supporters for $A$.

\begin{figure}[t]
    \centering
    \includegraphics[width=0.8\linewidth]{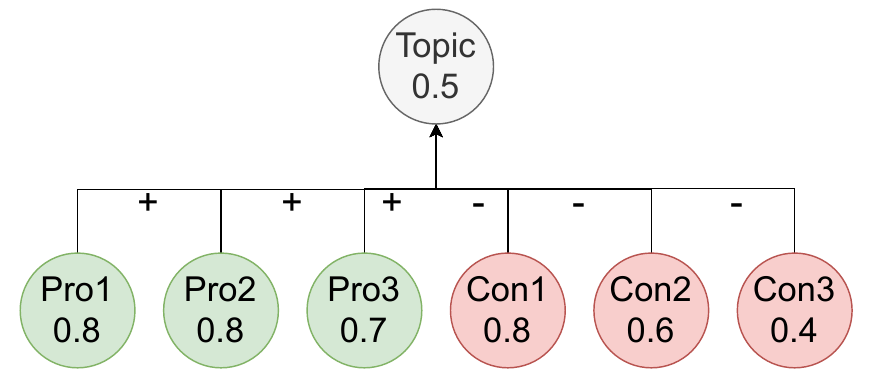}
    \caption{Example of a QBAF modelling a simple family debate about going to the zoo on Sunday. The central claim is the Topic argument; green and red nodes are pro and con arguments, respectively. Edges labelled \(+\) and \(-\) denote support and attack relations, and values denote base scores.}
    \label{fig_qbaf}
\end{figure}

% In the remainder, 
% unless specified otherwise, we assume as given a generic QBAF $\mathcal{Q}=\langle \mathcal{A}, \mathcal{R}^{-}, \mathcal{R}^{+}, \tau, w \rangle$ and $\mathcal{R}=\mathcal{R}^{-} \cup \mathcal{R}^{+}$.

%\begin{example}
For illustration, we %evaluate the QBAF in
apply DF-QuAD to the QBAF in Figure~\ref{fig_qbaf}. Since none of the three pro and three con arguments have attackers or supporters, their strengths are equal to their base scores shown in the figure.
Then, by applying aggregation and influence function, we have
$\sigma(Topic)=0.518$. \ft{This may be deemed to} supports the claim (topic argument) because its strength is greater than the neutral value of \(0.5\).
%\end{example}

DF-QuAD also satisfies several desirable properties \cite{rago2016discontinuity}. For example, it satisfies \emph{balance}  \cite{baroni2018many}: when a supporter and an attacker have the same strength, their effects on the topic argument cancel each other out (Pro2 and Con1 in Figure~\ref{fig_qbaf}). It also satisfies \emph{monotonicity} \cite{baroni2018many}: increasing the strength of a supporter (attacker), or adding a new supporter (attacker), cannot decrease (increase) the final strength of the topic argument. These properties 
%ensure that DF-QuAD satisfies reasonable rationality requirements when %determining the outcome of 
\ft{play an important role when DF-QuAD is used to judge} a debate.
% \todo{do we prove anything that needs them?}\todo{in the proof of contestability, but we have proven in ArgLLM paper}

% \begin{property}[Determinism]
% A judge \(J\) satisfies determinism iff, for any \(\mathcal{Q},\mathcal{Q}'\in\mathbb{Q}\), if
% \[
% \mathcal{Q}=\mathcal{Q}',
% \]
% then
% \[
% J(\mathcal{Q})=J(\mathcal{Q}').
% \]
% \end{property}

% Let \(d_{\mathbb{Q}}\) be a distance measure over QBAFs, and let
% \(d_{\mathrm{Rank}}\) be a distance measure over rankings.

% \begin{property}[Input Robustness]
% A judge \(J\) satisfies input robustness iff, for any \(\epsilon>0\), there exists \(\delta>0\) such that, for any \(\mathcal{Q},\mathcal{Q}'\in\mathbb{Q}\), if
% \[
% d_{\mathbb{Q}}(\mathcal{Q},\mathcal{Q}')<\delta,
% \]
% then
% \[
% d_{\mathrm{Rank}}\bigl(J(\mathcal{Q}),J(\mathcal{Q}')\bigr)<\epsilon.
% \]
% \end{property}

% Let \(d_{\mathcal{J}}\) be a distance measure over judge functions.

% \begin{property}[Judge Robustness]
% A judge \(J\) satisfies judge robustness iff, for any \(\epsilon>0\), there exists \(\delta>0\) such that, for any judge \(J'\in\mathcal{J}\) and any \(\mathcal{Q}\in\mathbb{Q}\), if
% \[
% d_{\mathcal{J}}(J,J')<\delta,
% \]
% then
% \[
% d_{\mathrm{Rank}}\bigl(J(\mathcal{Q}),J'(\mathcal{Q})\bigr)<\epsilon.
% \]
% \end{property}

% Neutrality: Adding a new argument that is irrelevant to the existing arguments, for example one with no support or attack relations to them, should not change the relative ranking among the original arguments.

\section{
%An Abstract Framework for Evaluating Post-hoc Debate Judgement
Abstract Framework}

In this section, we set out an abstract framework for evaluating \emph{post-hoc debate judgement}, whereby a \emph{judge} needs to determine the veracity of a claim given the opinions expressed by a number of agents in a debate.
%We focus on agents powered by LLMs, but without committing to any specific LLMs.
\ft{We do not impose any restriction on the agents and how they operate.}
Also, we mostly ignore how the debate %generating the agents' opinions 
is conducted, specifically whether it is single-round, with each
agent giving its opinion in isolation, as in \cite{jiang2023llm}, 
or multi-round, with agents engaging in conversations by providing their opinions incrementally and in response to other agents, as in \cite{du2024improving}. 
Our focus instead is on how and what an external  judge decides, given the agents' opinions, after they have been fully expressed by the agents.

\subsection{%Post-hoc Debate Judgement Framework
Formal set-up}

% Figure~\ref{fig:overview} provides an overview (left) and an illustration (right).\todo{not here rather in next section?}
%\todo{add two examples to the figure: one on question answering and single round, with straight outputs, the other on claim verification and multi-round and reasoning chains, both between two agents; in the first case the judgement is a ranking, in the second an output only..... refer to the figure throughout this section, in addition to adding references}

Formally, let \(\mathbb{A}\) be a set of \emph{agents}, and let \(\mathbb{O}\) denote the space of possible \emph{opinions} expressed in debates amongst the agents
about \emph{inputs} drawn from some set $\mathbb{I}$.
%For illustration,  inputs could be
\ft{Throughout the paper,} inputs 
 % questions in need of answers (as in Figure~\ref{fig:overview}, middle) 
 % or 
\ft{are \emph{claims} to be verified}% (as in Figure~\ref{fig:overview}, right)
.
% \todo{shall we commit to the latter}.
%To keep things at an abstract level, we do not impose any constraints on the format of opinions.
%Given our focus on debate, concretely, 
Opinions in $\mathbb{O}$ %will
consist of \ft{\emph{outputs}} $y \in \mathbb{Y}$
as well as of \emph{explanations} therefor. \ft{Given our focus on inputs as claims, concretely outputs amount to \emph{stances} on the input claims' veracity.}
\ft{Given our focus on debates, concretely, explanations
consist of \emph{pros} and \emph{cons}, e.g. as in \cite{freedman2025argumentative}}.
% In practice, as in \cite{}\todo{add},
% opinions may be straight \emph{outputs} from a set $\mathbb{Y}$, for example, when the inputs are questions, answers to questions, as in Figure~\ref{fig:overview}, middle.
% Alternatively, as in \cite{}\todo{add}, they may be \emph{outputs accompanied by an explanation}, i.e. 
% pairs drawn from $\mathbb{Y}\times \mathbb{E}$ for $\mathbb{E}$ a space of possible explanations,
% for example,
% when the inputs are claims, reasons for veracity labels, as
% in Figure~\ref{fig:overview}, middle.
These %explanations 
pros and cons may, for instance, be simply chunks of text, as in \cite{freedman2025argumentative} \ft{and as illustrated in the toy example in Section~\ref{sec:back}}, or take the form of structured reasoning chains towards or against the outputs, as in \cite{wei2022chain}%, or pros and cons of the outputs that informed them, as in \cite{}\todo{add}
, amongst others.
%\emph{ In this paper we focus on structured opinions  (outputs plus explanations).}
% 
Note that the  pros and cons, as explanations for the outputs,  may or may not faithfully reflect the internal reasoning for the outputs by the agents producing them, as studied in \cite{jacovi2020towards}, but ascertaining this is outside the scope of our paper, as 
our focus is on the behaviour of a judge deciding based on the opinions it sees, independently of the agents' generation processes. 

Let \(\mathbb{J}\) denote the space of possible \emph{judgements} by a judge. 
For a fixed number \(n \geq 1\) of agents $A_1, \ldots, A_n \in \mathbb{A}$ and an input $I \in \mathbb{I}$, an \emph{opinion profile} (for $I$) is a sequence%\todo{AR: shall we not also include a mapping from agents to opinions to identify speakers? FT: IMPLICIT NOW THAT AGENTS ARE FIRST CLASS CITIZENS AND WE CAN USE INDEX}
\[
\mathcal{O}=(O_1,\ldots,O_n)\in\mathbb{O}^n,
\]
where \(O_i = (y_i, e_i)\) is the opinion produced by  $A_i$ on $I$, $y_i$ is the stance and $e_i$ comprises the pro and con explanations for $y_i$.
Abstractly, a \emph{judgement method} can be represented as %a mapping NOT NECESSARILY A MAPPING=FUNCTION
\[
c:\mathbb{I}\times \mathbb{O}^n\rightarrow\mathbb{J}
\]
where, given an opinion profile \(\mathcal{O}\), \(c(I,\mathcal{O}) = (y^*, e^*) = J \in \mathbb{J}\) denotes the final judgement returned by the %judgement 
method.
\ft{We
do not impose that judgements are drawn from the same
space as agents’ opinions and, 
in general, $\mathbb{J}$ may be different from $\mathbb{O}$}. 
\iffalse 
$\mathbb{O} \neq \mathbb{J}$, although they may coincide in practice.
In any case, as in the case of opinions, judgements may also take various forms. For instance, they may amount to
%choice amongst all agent-generated opinions or a 
% ranking amongst all agent-generated opinions (as in Figure~\ref{fig:overview}, middle)
% or 
a choice amongst stances (as in Figure~\ref{fig:overview}, right). 
%If the opinions have a structure (output and explanation), 
Judgements may also
%amount to outputs alone or  ranking between outputs,  with and without 
be equipped with explanations (as in Figure~\ref{fig:overview}, right). 
\emph{In this paper we focus on judgements equipped with explanations.}
% Thus, we assume that,
% given an opinion profile \(\mathcal{O}\), \(c(I,\mathcal{O}) = (y^*, e^*) = J \in \mathbb{J}\).
However, in this section, we do not impose any restrictions as to what explanations ($e$) may amount to.
Moreover, for reasons of generality, we do not impose that $(y,e) \in \mathbb{J}$
is such that $y \in \mathbb{Y}$ and/or $e \in \mathbb{E}$ (namely we do not impose that judgements are drawn from the same space as agents' opinions), although in practice this may be the case. 
\fi
%For example
Specifically, a judgement method may return an explicit uncertainty or abstention outcome when the agents' opinions do not support a sufficiently confident choice among the outputs in $\mathbb{Y}$.

\subsection{Properties}
In the remainder of this section we focus on general formal properties that judgement methods may satisfy  (besides standard quantitative measures in the downstream task of interest, e.g. accuracy). 
Note that we do not necessarily  see all these properties as universally desirable; rather, we see them as formal guidelines to discriminate amongst different candidate judgement methods and to guide their choice.

In formalising the properties, we  assume as given some input $I \in \mathbb{I}$:
all opinion profiles $\mathcal{O}$ will be implicitly intended to be for $I$.
Also, in order to assess similarities, we make use of distance measures, as follows:

%\begin{itemize}
   %\item 
   $\bullet$ a distance measure   
   \(d_{\mathbb{O}^{n}}\) over opinion profiles;
   
   %\item 
   $\bullet$ 
   a distance measure \(d_{\mathbb{Y}}\) over stances in judgements;
   % \item a distance measure \(d_{\mathbb{E}}\) over explanations in judgements;
   
   %\item 
   $\bullet$ a distance measure \(d_{\mathbb{C}}\) over judgement methods. 
%\end{itemize}

\noindent We keep %the precise form of 
these distance measures abstract, as suitable choices therefor depend on the particular types of opinion profiles, stances, explanations, and judgement methods being considered. For instance, categorical outputs may be compared using a discrete metric, whereas textual opinions or explanations may be compared using semantic distance measures.

We now introduce %a collection of 
our properties for evaluating post-hoc debate judgement methods, with a particular focus on the components of the judgements they generate.

\begin{property}[Determinism]
A method \(c\in \mathbb{C}\) satisfies \emph{determinism} iff, for every  opinion profile 
\(\mathcal{O}=(O_1,\ldots,O_n)\in\mathbb{O}^n\), there exists a unique \(y^*\) such that
\(
c(I,\mathcal{O})=(y^*,e^*)
\)
for some explanation \(e^*\).
\end{property}

Intuitively, determinism requires that the judgement should be the same if identical input and opinion profile are given.
This property is provably satisfied if $c$ is a function, but may or may not be satisfied 
when $c$ is defined stochastically%, e.g. in terms of probabilities and thresholds
.

\begin{property}[Permutation Independence]
A method \(c\in \mathbb{C}\) satisfies \emph{permutation independence} iff, for every opinion profile \(\mathcal{O}\!=\!(O_1,\ldots,O_n)\in\mathbb{O}^n\)  and every permutation \(\mathcal{O}'\) of \(\mathcal{O}\),
if $c(I,\mathcal{O}) = (y^*, e^*)$ and $c(I,\mathcal{O}^{\prime}) = (y^{\prime}, e^{\prime})$, then $y^*=y^{\prime}$. 
\end{property}

Intuitively, permutation independence \ft{requires that the order in which the agents present their opinions does not influence the judgement. It } embeds a form of
\emph{anonymity}, requiring the judge to assess opinions based on their
content rather than on the identities of the agents expressing them.
This is desirable when agent identities may introduce bias into the judgement. However, it may be undesirable when agents differ in expertise or reliability, since the judge may reasonably take such differences into account.
\ft{Note that this formulation of the property  imposes no restriction on the  explanations generated by the judge, but could be extended to do so.}

\begin{property}[Profile Robustness]
A method \(c\in \mathbb{C}\) satisfies \emph{profile robustness} iff, for any \(\epsilon>0\), there exists \( \delta>0\)\ such that for any $\mathcal{O},\mathcal{O}' \in\mathbb{O}^n$, if $ d_{\mathbb{O}^n}(\mathcal{O},\mathcal{O}')<\delta
$, \(c(I,\mathcal O)=(y,e)\), and
\(c(I,\mathcal O')=(y',e')\), then
\(d_{\mathbb Y}(y,y')<\epsilon\).
\end{property}

Intuitively, profile robustness requires that small changes to opinion
profiles do not cause large changes in the stances of the resulting judgements. In settings
with discrete outputs, such as binary claim verification, this amounts
to requiring the output to remain unchanged under sufficiently small
changes to the opinion profile.

\iffalse
\paragraph{Judge Robustness version 1 (FT - DROP?)}
%Let \(d_{\mathbb{O}^n}\) be a distance function over opinion sequences, and let \(d_{\mathbb{Y}}\) be a distance function over judgements.
A method \(c\) satisfies \emph{judge robustness} iff, for every method $c'$, for every 
%sequence of opinions \(\mathcal{O}\in\mathbb{O}^n\)  and every
\(\epsilon>0\) there exists \(\delta>0\) such that 

\[
\forall \mathcal{O}\in\mathbb{O}^n \quad 
d_{\mathbb{Y}} (c(\mathcal{O}),c'(\mathcal{O}))<\epsilon \quad \Rightarrow \quad [
\forall \mathcal{O}' \in\mathbb{O}^n
 [d_{\mathbb{O}^n}(\mathcal{O},\mathcal{O}')<\delta \Rightarrow
d_{\mathbb{Y}}(c(\mathcal{O}),c'(\mathcal{O}'))<\epsilon
]]. \]

sequence of opinions \(\mathcal{O}\in\mathbb{O}^n\)  and every \(\epsilon>0\), there exists \(\delta>0\) such that, for every sequence of opinions \(\mathcal{O}'\in\mathbb{O}^n\),
\[
d_{\mathbb{O}^n}(\mathcal{O},\mathcal{O}')<\delta
\quad \Rightarrow \quad
d_{\mathbb{Y}}(c(\mathcal{O}),c(\mathcal{O}'))<\epsilon.
\]
Basically,...if the judge changes very little, then when applied to %the same
a similar opinion the judgement does not change or changes little...
\fi

\begin{property}[Judge Robustness]
A method \(c \in \mathbb{C}\) satisfies \emph{judge robustness} iff, %$\forall \mathcal{O}\in\mathbb{O}^n$,
for any \(\epsilon>0\), there exists $\delta>0$ such that for any
\(\mathcal{O}\in\mathbb{O}^n\), for any \(c'\in\mathbb{C}\), if $d_{\mathbb{C}}(c,c')<\delta$, \(c(I,\mathcal O)=(y,e)\), and
\(c'(I,\mathcal O)=(y',e')\), then
\(d_{\mathbb Y}(y,y')<\epsilon\).
\end{property}

Intuitively, judge robustness requires that small changes to the judging method itself do not lead to large changes in the %stances of the 
resulting judgements, when the same opinion profile is judged.
This property is desirable when small implementation-level changes to the judging method, such as minor prompt %or model 
variations \ft{if the judge is an LLM}, should not materially affect the resulting judgement. However, it may be less desirable for cases close to a decision boundary, where even small changes to the judging method may justifiably lead to a change in stance.

\begin{property}[Non-hallucination]
A method \(c \in \mathbb{C}\) satisfies \emph{non-hallucination} iff, for every opinion profile \(\mathcal{O}=(O_1,\ldots,O_n)\in\mathbb{O}^n\),
if $c(I,\mathcal O)=(y^*,e^*)$, then there exists \(i\in\{1,\ldots,n\}\) such that $y^* = y_i$.
\end{property}

Intuitively, non-hallucination requires the judge to select a stance proposed by at least one agent, rather than introducing a new stance. This may be desirable when the judge is expected to choose among a discrete set of candidate stances, as it prevents the introduction of a stance that has not been proposed or defended during the debate. However, it may be less desirable when all agents provide incorrect stances, since it prevents the judge from returning a correct alternative.

% Intuitively, this property requires that the judgement coincides exactly with at least one of the opinions.
% This requires that the space of agents' opinions is included in the space of judgements. Besides this syntactic restriction, it is a very strong property, e.g. violated in Figure~\ref{fig:overview}, right.   
% We envisage that relaxations of this property may be more desirable, whereby the judgement is required to be similar (as concerns outputs and/or explanations), rather than identical, to one of the opinions, but leave this to future work.
%\todo{AR: I like this one but it may only be suitable if the opinion space is discrete? what if the judgement is an aggregation of the opinions, e.g. in the latent space?}

\begin{property}[Judge Unanimity]
A method \(c \!\!\in\!\! \mathbb{C}\) satisfies \emph{judge unanimity} iff, for every 
opinion profile \(\mathcal{O}\!=\!(O_1, \ldots,\) \(\!O_n)\!\in\!\mathbb{O}^n\), if there exists  
 $y^*$ such that $y^*\!=\!y_i$ for all $%1\leq i \leq n
i \!\in\! \{1, \!\ldots,\! n\}$, then
$c(I,\mathcal{O})\!=\!(y^*,e)$ for some explanation $e$.
\end{property}

Intuitively, judge unanimity requires the judge to adopt the common stance when all agents agree. It is weaker than non-hallucination, since it applies only to unanimous profiles. Together, the two properties more strongly constrain the judge to follow the agents' stances. However, like non-hallucination, judge unanimity is undesirable when all agents are collectively mistaken or share the same systematic bias, as the judge would then reproduce their incorrect consensus.

\ft{The final property we propose requires a new notion of \emph{supportiveness}: intuitively, }
%\begin{definition}[Supportiveness]
for opinion profiles \(\mathcal O, \mathcal O' \in \mathbb O^n\), we write
\(
\mathcal O\preceq_{\mathrm{sup}}\mathcal O'
\)
iff \(\mathcal O'\) is overall at least as supportive as \(\mathcal O\);
similarly, for judgement outputs \(y,y'\), we write
\(
y\preceq_{\mathrm{sup}}y'
\)
iff \(y'\) is at least as supportive as \(y\).
%\end{definition}
\ft{Specifically, }
%Intuitively, 
for opinion profiles,
\(\mathcal O\preceq_{\mathrm{sup}}\mathcal O'\) may result from a
shift towards a more supportive stance, the addition or strengthening
of pro explanations, or the removal or weakening of con explanations.
For judgement outputs,
in binary claim verification, for example, 
\(\mathsf{false}\preceq_{\mathrm{sup}}\mathsf{true}\).

\begin{property}[Contestability]
A method \(c\in\mathbb C\) satisfies \emph{contestability} iff, for every
\(\mathcal O,\mathcal O'\in\mathbb O^n\), if
\(
\mathcal O \preceq_{\mathrm{sup}} \mathcal O'
\),
$c(I,\mathcal O) = (y,e)$
and
$c(I,\mathcal O') = (y',e')$,
then
$y\preceq_{\mathrm{sup}}y'$.
\end{property}

Intuitively, contestability requires a judgement method to respond
consistently to directed changes in the opinion profile: making the
profile overall more supportive should not result in a less supportive
judgement, as in \cite{freedman2025argumentative}. In concrete settings, contestability may additionally be
studied in terms of the minimal modification required to obtain a
desired judgement, as in \cite{yin2024ceqarg}.
This property provides users with a mechanism to challenge and potentially correct erroneous judgements, but the same mechanism may also be exploited to manipulate %the resulting judgement
judgements. In practice, this risk may be mitigated by restricting contestation to authorised users \cite{AyoobiKR}.

% It is a weak property, as it only envisages changes `local' to individual opinion profiles (so that wrong judgements for different profiles can be accommodated by different changes). 

% \paragraph{Non-hallucination of explanations...}
% When included in the judgements, the rationales may be drawn from those in the opinions or not,  
% ...

\iffalse
\todo{maybe define a function}
\todo{to be decided: cognitive size, identity to the mechanism}
\paragraph{Faithfulness% of Judgement Explanations
.}
A method \(c \in \mathbb{C}\) satisfies \emph{faithfulness}  iff, for every opinion profile \(\mathcal{O}=(O_1,\ldots,O_n)\in\mathbb{O}^n\), 
if $c(I,\mathcal{O})=(y^*,e^*) \in \mathbb{J}$, 
then
$y^*$ is reproducible from $E$.

This property is formulated in  terms of a generic notion of `reproducibility'....\todo{ what does reproducible mean formally? the same $c$? another $c'\in \mathbb{C}$? an external algorithm?}.
%if the opinions include a rationale...
Intuitively, a method satisfies this property if %explanation for a judgement is faithful if 
the judgements it generates can be reconstructed from the explanations  thereof together 
%with the agents' opinions (the full opinions? just their output component?).
the original input.

%\paragraph{Cognitive Tractability????}
%\cite{DBLP:conf/aaai/CyrasLMT19}
%cognitive tractability, in the sense that each explanation pertaining to schedule S and presented to the user should be polynomial in the size of S.

% \paragraph{Contestability of Rationales for judgements}

\fi

\begin{figure}[t]
    \centering
\includegraphics[width=0.9\linewidth]{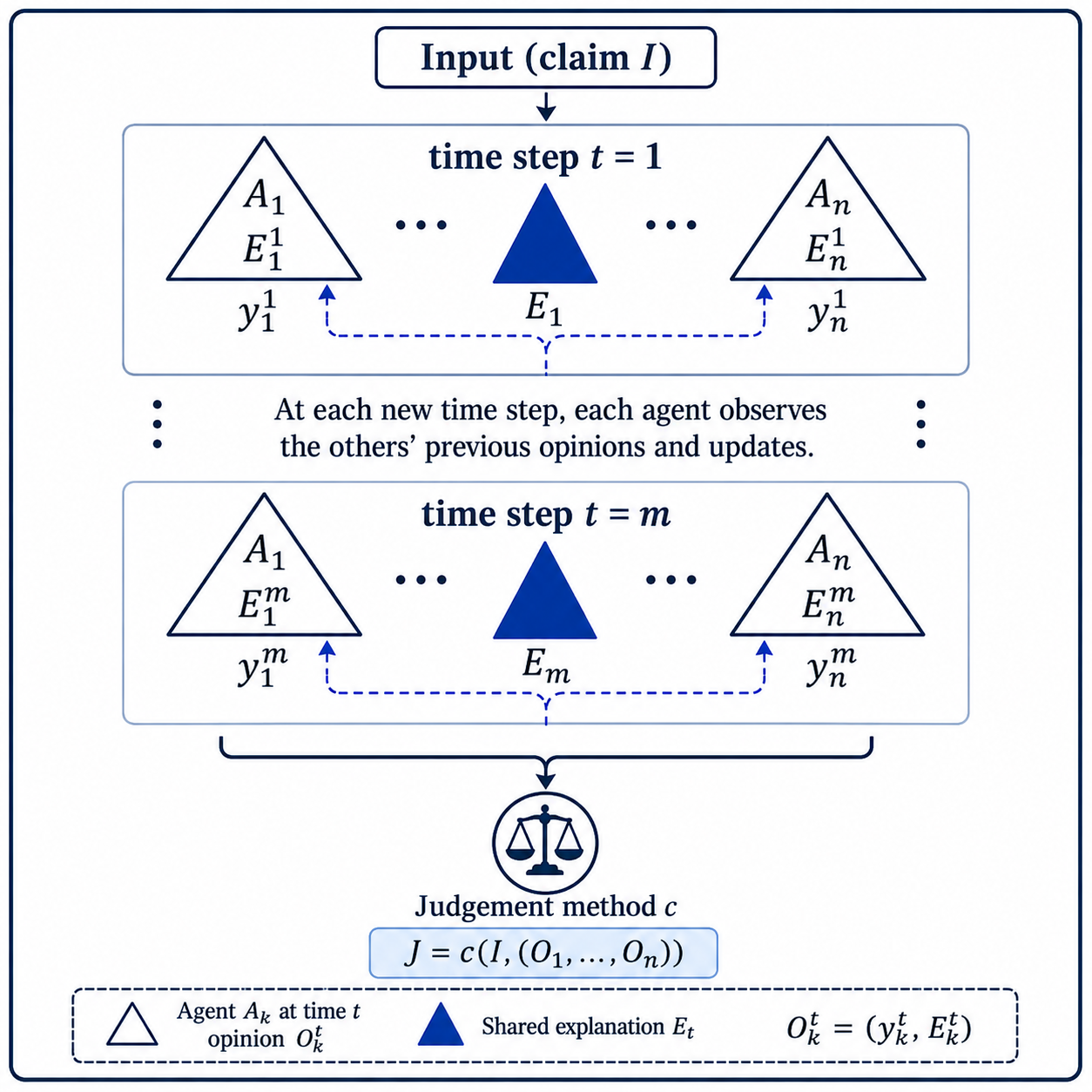}
    \caption{
    Overview of \emph{post-hoc debate judgement}: $n \geq 1$ agents ($A_1, \ldots, A_n$) debate an \emph{input} (claim) $I$ over $m \geq 1$ time steps. At time step $t \in \{1, \ldots,m\}$, agent $A_k$ ($k \in \{1, \ldots, n\}$) produces \emph{opinion} $O_k^t$ made up of explanation (pros and cons) $E_k^t$ and output (stance) $y_k^t$; then, a judge gives a \emph{judgement} (stance) $J$ 
    %based on the  agents' opinions. 
    determined by a \emph{judgement method} $c$ applied to the input and
    $O_1, \ldots, O_n$, drawn from the debate. At each time step $t$, agents may contribute to a shared explanation $E_t$, \ft{feeding into $O_1, \ldots, O_n$.}
    %\ft{$J$ may focus, specifically, on aggregations $E_1, \ldots, E_m$ of the agents' explanations at the respective time steps.}
At each time step %$k \in \{2, \ldots, m\}$
, each agent sees the other agents' opinions at the previous time step %$k-1$ 
and can adjust its own opinion accordingly. 
    % Left:  overview of \emph{post-hoc debate judgement}: $n \geq 1$ agents ($A_1, \ldots, A_n$) debate an \emph{input} claim $I$ by providing \emph{opinions} ($O_1, \ldots, O_n$); then, a judge gives a  \emph{judgement} $J$ based on the  agents' opinions. Abstractly, $J$ is determined by a \emph{judgement method} $c$ applied to the input and agents' opinions.
    % Middle: a simple illustration, with three agents, where the input is a question, the agents' opinions are answers, and the judgement is a ranking over answers. Right: another simple illustration, with two agents, where the input is a claim, the agents' opinions are veracity labels with an explanation, and the judgement is a veracity label with an explanation (all explanations are in brackets).
    }
    \label{fig:overview}
\end{figure}
\section{%Experimental
\ft{Evaluation} Setting}
\label{sec:experimental-setting}

We consider two %experimental 
concrete scenarios for the evaluation of post-hoc debate judgement.
In both scenarios, we use LLM-driven agents (three in the experiments, empowered by %gpt-4o-mini, Llama-3.3-70B-Instruct-Turbo, and Qwen3.5-9B respectively
different LLMs), using  prompts \ft{(as detailed in the Supplementary Material (SM)%\todo{generic so no problem with submitting it late with different appendices ...first time we introduce SM give the full Supplementary Material}
)} to generate stances on input claim veracity and/or explanations thereof, made of pros and cons (but a different number in the two evaluation scenarios). 
In both scenarios, 
claims' stances are real numbers in $[0,1]$%, representing the degrees to which an agent supports  the claims being true, where a score greater than or equal to 0.5 is treated as support and a score below 0.5 is treated as opposition \ft{(for and to veracity, respectively)}
. 
Also, in both scenarios, 
the pros and cons amount to \emph{textual arguments} with associated \emph{base scores} (in $[0,1]$) for the arguments, generated after the arguments by the agents themselves, and expressing %the agents' 
\ft{their} confidence in the arguments,
in the spirit of \cite{wachsmuth-etal-2024-argument}.

Both scenarios are instances of Figure~\ref{fig:overview}.
In the first, \emph{single-turn} scenario, $m=1$, namely
agents generate their opinions independently, in one go.
In the experiments, for this scenario we focus on opinions where explanations consist of a single pro and a single con.  
%We refer to this scenario as \emph{single-turn}.
In the second, \emph{multi-turn} scenario, $m > 1$ \ft{(in the experiments we set $m=3$)}
and agents can refine or even change completely their opinions over time.
In the experiments, for this scenario we focus on opinions with
any number of pros and cons in the explanation component. 
%We refer to this scenario as \emph{multi-turn}.

For both scenarios, we consider two variants for the generation of the stance and explanations components of opinions by each agent: 
the stance is generated before or after the explanation is generated.
We refer to the former variant as \emph{prior} and to the  latter variant as \emph{post}. To assess the usefulness of explanations in debates, we also consider
a third variant % \todo{CHECK: also for the multi turn scenario?} 
alongside the prior and post variants, where no stance score is generated. 
We refer to this third variant as \emph{none}.

In all scenarios and variants, we consider two families of judges:
\emph{LLM-as-a-judge} (in all the experiments we use gpt-4o as its  underpinning LLM, \ft{see prompts in the SM}) and 
\emph{semantics-as-a-judge}
(in all the experiments in the paper we use the commonly adopted DF-QuAD~\cite{rago2016discontinuity} as the semantics).
Semantics-as-a-judge maps explanations to QBAFs: input claims become topic arguments, pros and cons become supporting and attacking arguments, their base scores are as in the explanations, and the topic arguments' base score may be the neutral  0.5 (in the \emph{none} variant) or the claim's stance (in the \emph{prior} variant). 
%The judge predicts a positive stance if DF-QuAD computes a value of 0.5 or over, and a negative stance otherwise.

In all scenarios and variants,  
each judge is presented with information from (the last time step in) the debate.
We consider two settings as concerns this information:
in the \emph{combined} setting, the judge is given a single explanation made of all pros and cons in the explanations from all agents (as in the toy example in Section~\ref{sec:back});
%In contrast, 
in the \emph{separate} setting,   
%the judge 
it is given the individual explanations from the agents separately.
%where each agent forms one QBAF, and a combined QBAF, where all three agents are included in one QBAF. 

Finally, in the multi-turn scenario, we consider two \ft{sub-scenarios}: a \emph{private} one where each agent generates its own explanations, but can adjust them depending on the explanations from the other agents, and a \emph{shared} sub-scenario where the agents collaboratively construct a single explanation ($E_1, \ldots, E_m$ in Figure~\ref{fig:overview}). However, note that the agents maintain their own private stances in both sub-scenarios.

For each evaluation scenario and variant thereof, we measure judgement accuracy as well as  satisfaction of the properties given earlier.
To measure accuracy,
we assume that judges predict a positive judgement when computing a stance of 0.5 or over, and a negative judgement otherwise.

\section{Experiments}

\subsection{Data and Configurations}
\noindent\textbf{Data generation}
%This experiment uses 
We use the claims labelled as True or False in the claim verification dataset from \cite{freedman2026neurosymbolic}. The resulting dataset contains 500 claims, with approximately balanced numbers of True and False instances. For each claim, we construct three LLM-based agents to simulate opinions from different sources ($A_1$: GPT-4o-mini, $A_2$: Llama-3.3-70B-Instruct, $A_3$: Qwen/Qwen3.5-9B). 
With 500 claims and 3 agents, the final dataset for the single-turn scenario contains 1500 agent-level opinions, which are then used in both LLM-as-a-judge and Semantics-as-a-judge experiments. 
%\todo{Added multi-turn, to cross-check}
In the multi-turn scenario, we consider the two sub-scenarios (\emph{private} and \emph{shared}) described in Section \ref{sec:experimental-setting}%: a \emph{private} scenario where each agent generates its own explanations, but can adjust them depending on the explanations from the other agents, and a \emph{shared} scenario where the agents collaboratively construct a single explanation. However, the agents maintain their own private stances in both scenarios
. 
Since we run the debate for three time steps, this results in $500 \times 3 \times 3 = 4500$ agent-level stances for each scenario, $4500$ private agent explanations in the private scenario and $1500$ shared explanations in the shared scenario. 
Except where otherwise stated, we only consider stances and explanations from the final round of the debate during judging.

\begin{table}[t]
\centering
\small
\caption{Judgement accuracy of LLM-based and
semantics-based judges under different scenarios, settings and variants.
\ft{We also report accuracy for {\bf T}rue and {\bf F}alse  judgements.
Overall best results are in {\bf bold}.}
% Judgement accuracy of the LLM-as-a-judge and
% QBAF-semantics-as-a-judge methods under different variations.
% For LLM, \emph{Setting} indicates whether agent identities are
% anonymized, while \emph{Variant} denotes the stance-score setting.
% For QBAF, \emph{Setting} denotes the QBAF structure, while
% \emph{Variant} denotes the base-score setting.
% All accuracies are reported as percentages, with the best results
% for each method shown in bold.
% Judgement accuracy of the LLM-as-a-judge and QBAF-semantics-as-a-judge methods at the end of a third (final) round of multi-turn debate under different variations. Column names have the same meaning as in Table \ref{tab:judgement-accuracy}, except for \emph{Scenario}, which denotes the debate explanation type.
}
\label{tab:judgement-accuracy}
{\setlength{\tabcolsep}{3pt}
\begin{tabular}{lllccc}
\toprule
\textbf{Scenario} & \textbf{Setting} & \textbf{Variant} & \textbf{Acc.} & \textbf{T-Acc.} & \textbf{F-Acc.} \\
\midrule
\multicolumn{6}{l}{\textit{LLM as a judge}} \\
\midrule
Single-turn & No anon. & None  & 72.20 & 54.22 & 90.04 \\
Single-turn & No anon. & Prior & 72.80 & 53.41 & 92.03 \\
Single-turn & No anon. & Post. & 74.60 & 57.03 & 92.03 \\
Single-turn & Anon.    & None  & 72.00 & 53.41 & 90.44 \\
Single-turn & Anon.    & Prior & 74.20 & 57.03 & 91.24 \\
Single-turn & Anon.    & Post. & 75.00 & 58.63 & 91.24 \\
\midrule
Multi-t. private & No anon. & None & 73.40 & 53.01 & 93.63 \\
Multi-t. private & No anon. & Prior & 72.40 & 51.00 & 93.63 \\
Multi-t. private & No anon. & Post. & 73.00 & 53.01 & 92.83 \\
Multi-t. private & Anon. & None & 74.00 & 56.22 & 91.63 \\
Multi-t. private & Anon. & Prior & 74.00 & 55.02 & 92.83 \\
Multi-t. private & Anon. & Post. & 74.00 & 56.22 & 91.63 \\
Multi-t. shared & No anon. & None & 73.60 & 54.22 & 92.83 \\
Multi-t. shared & No anon. & Prior & 74.00 & 53.82 & \textbf{94.02} \\
Multi-t. shared & No anon. & Post. & 74.00 & 54.62 & 93.23 \\
Multi-t. shared & Anon. & None & 73.40 & 53.82 & 92.83 \\
Multi-t. shared & Anon. & Prior & 74.80 & 56.22 & 93.23 \\
Multi-t. shared & Anon. & Post. & 74.00 & 54.22 & 93.63 \\
\midrule
\multicolumn{6}{l}{\textit{Semantics-as-a-judge}} \\
\midrule
Single-turn & Separate & 0.5         & 68.40 & 48.59 & 88.05 \\
Single-turn & Separate & Prior      & 74.40 & 59.84 & 88.84 \\
Single-turn & Combined & 0.5         & 64.60 & 53.82 & 75.30 \\
Single-turn & Combined & Avg. Prior & \textbf{76.00} & 65.06 & 86.85 \\
\midrule
Multi-t. private & Separate & 0.5 & 70.40 & 49.40 & 91.24 \\
Multi-t. private & Separate & Prior & 75.40 & 62.25 & 88.45 \\
Multi-t. private & Combined & 0.5 & 72.20 & \textbf{67.07} & 77.29 \\
Multi-t. private & Combined & Avg. Prior & 75.40 & 66.67 & 84.06 \\
Multi-t. shared & Separate & 0.5 & 69.80 & 53.01 & 86.45 \\
Multi-t. shared & Separate & Prior & 73.40 & 61.04 & 85.66 \\
Multi-t. shared & Combined & 0.5 & 69.80 & 53.01 & 86.45 \\
Multi-t. shared & Combined & Avg. Prior & 73.40 & 61.04 & 85.66 \\
\bottomrule
\end{tabular}
}
\end{table}

\noindent \textbf{Configurations}
For \textbf{LLM-as-a-judge}, we use GPT-4o as the judge model with a prompt-based setup. We consider both \emph{non-anonymised} and \emph{anonymised} %inputs,\todo{not sure what this means- which agent?}
settings \ft{as concerns agents' identity}. The anonymised setting is aligned with \textbf{semantics-as-a-judge} (since argumentation semantics does not use agent identity information). 
%We test three input variants: no stance score, prior stance score, and posterior stance score, corresponding to different levels of stance information available to the judge.\todo{repeated from experimental set-up?}
For the latter judge, we use the commonly adopted DF-QuAD \ft{(but report on the use of QE \cite{Potyka18} as an alternative semantics in the SM)}.
%we consider two variants: one forming \emph{separate} QBAFs from the opinions of each agent and one forming a single, \emph{combined} QBAF for all agents. 
In the multi-turn \emph{shared}  scenario, the %separate QBAFs 
\ft{private explanations all amount} to the same, shared explanation but may differ in the base scores on the topic argument depending on the agents' private stance scores. 
To align with the \ft{use of LLM-as-a-judge}, we test two variants: \emph{none} and \emph{prior} \ft{(see Section~\ref{sec:experimental-setting})}. 
%In the first variant, the topic argument   uses a neutral base score of 0.5. 
In the %prior-stance 
second variant, 
for the \emph{separate} setting, the topic argument's base score is the agent’s initial stance score before it has generated an explanation \ft{(we do not consider the {\em post} variant for semantics-as-a-judge, as the semantics aggregates the arguments and base scores into a final judgement)}; for the \emph{combined} setting, we use the average of the three agents’ prior stance scores as the topic argument's base score. %We do not use posterior stance scores for QBAF, since QBAF reasoning itself aggregates the arguments and base scores into a final judgement.
%\todo{Do we need to talk about how we determine the final judgement for each QBAF variant or is this somewhere else? ADDED IN SECT 5} 
%All QBAF experiments use the commonly adopted DF-QuAD gradual semantics~\cite{rago2016discontinuity}.

\subsection{Accuracy}
Table~\ref{tab:judgement-accuracy} shows that LLM-as-a-judge and semantics-as-a-judge achieve comparable performance, with average accuracies above 70\% for both families of judges. The best overall accuracy is obtained by semantics-as-a-judge (76.00\%), slightly higher than the best LLM-as-a-judge result (75.00\%). In terms of class-wise performance, the highest True accuracy is achieved by 
%the same QBAF configuration (65.06\%)
\ft{the same judge (67.07\%)}, whereas the highest False accuracy is achieved by LLM-as-a-judge \ft{(94.02\%)}, possibly because the LLM can leverage its background knowledge to identify implausible or factually inconsistent claims. In addition, across both judge families, incorporating stance-score information \ft{(prior and post variants)} consistently improves accuracy over the corresponding no-stance or neutral-initialization settings \ft{(none variants)}, suggesting that stance scores may provide an agent-level global judgement signal that complements the local evidence captured by pro/con arguments and their base scores.

In the multi-turn scenario, semantics-as-a-judge using private explanations and prior stance scores achieves the best overall accuracy of $75.40\%$, slightly better than the best performing LLM-as-a-judge variant with an accuracy of $74.80\%$. However, both of these agent and judge configurations underperform the best-performing configurations from the single-turn scenario, suggesting that an extended debate over multiple turns is not helpful for achieving higher performance on the considered task \ft{(in line with prior work~\cite{humulti2025})}. 
Also, compared to the single-turn setting, 
in the multi-turn, providing the prior and posterior stance scores to the LLM judge tended to have a lower and more variable effect, but still resulted in consistent improvements for semantics-as-a-judge. While the private explanations sub-scenario was associated with better performance when used with semantics-as-a-judge compared to the shared explanations sub-scenario, there was no substantial difference between the two for LLM judges.

\subsection{Properties}

\ad{We provide experimental results from the evaluation of different properties below. Except where otherwise stated, we focus on configurations involving prior stance scores and no anonymisation/separate QBAFs.} \ft{(The SM includes additional experiments.)}

\noindent \textbf{Determinism}
\ad{We evaluate determinism by running matching LLM-as-a-judge configurations three times at the sampling temperature of 0.0}
%To examine determinism, we run matching LLM-as-a-judge configurations three times 
and check whether the resulting accuracies are identical.
%We first use the default temperature value of 1.0. As shown in Table~\ref{tab:llm-determinism}, the three runs produce different results. We then set the temperature to 0.0 to reduce sampling randomness. 
\ad{The resulting accuracy values vary between 73.20\% and 73.40\%}. This suggests that the LLM-as-a-judge method is not fully deterministic, even under zero-temperature decoding. A possible reason is that LLM API inference may involve implementation-level nondeterminism, such as batching. In contrast, the semantics-as-a-judge consistently produces exactly the same results, because QBAF semantics is defined by fixed analytical functions and is thus deterministic% by definition
.

\noindent \textbf{Permutation Independence}
We examine permutation independence by varying the order in which the agents' opinions are presented to the LLM judge. The results show that the LLM-as-a-judge method is sensitive to agent ordering: the overall accuracy varies \ad{from $72.60\%$ to $73.40\%$}, indicating a violation of permutation independence. In contrast, the semantics-as-a-judge satisfies this property, since QBAF semantics \ad{does not depend on agent ordering}.
%depends only on the set of arguments, support/attack relations, and base scores, rather than the order in which agents are listed.

\begin{table}[t]
\centering
\small
\caption{Profile-robustness results for semantics-as-a-judge
and LLM-as-a-judge, with pro scores increased by \(0.1\).
% \emph{Pro \(+0.1\), capped} denotes increasing pro scores by \(0.1\),
% capped at their maximum value. All accuracies are reported as
% percentages, and \(\Delta\) Acc. is reported in percentage points.
}
\label{tab:profile-robustness}
{\setlength{\tabcolsep}{4pt}

\begin{tabular}{lllcc}
\toprule
\textbf{Method} & \textbf{Scenario} & \textbf{Profile}
& \textbf{Acc.} & \(\boldsymbol{\Delta}\)\textbf{ Acc.} \\
\midrule
LLM & Single-turn & Baseline
& 72.80 & \(+0.00\) \\
LLM & Single-turn & Pro \(+0.1\), capped
& 75.20 & \(+2.40\) \\
LLM & Multi-t. private & Baseline & 72.40 & \(+0.00\) \\
LLM & Multi-t. private & Pro \(+0.1\), capped & 74.00 & \(+1.60\) \\
\midrule
QBAF & Single-turn & Baseline
& 74.40 & \(+0.00\) \\
QBAF & Single-turn & Pro \(+0.1\), capped
& 74.60 & \textbf{+0.20} \\
QBAF & Multi-t. private & Baseline & 75.40 & \(+0.00\) \\
QBAF & Multi-t. private & Pro \(+0.1\), capped & 75.00 & \(-0.40\) \\
\bottomrule
\end{tabular}
}
\end{table}

\begin{table}[t]
\centering
\small
\caption{Judge-robustness results for the QBAF-based
and LLM-based judges, following changes to their parameters.
%with prior stance scores and named agents/separate QBAFs. \emph{Temp.} denotes the sampling temperature used by the LLM while \emph{Conserv.} denotes the conservativeness parameter of the QBAF semantics. $\Delta$ Acc. denotes accuracy change compared to $1.0$ baselines.
}
\label{tab:judge-robustness}
{\setlength{\tabcolsep}{4pt}
\begin{tabular}{lllcc}
\toprule
\textbf{Method} & \textbf{Scenario} & \textbf{Temp. / Conserv.}
& \textbf{Acc.} & \(\boldsymbol{\Delta}\)\textbf{ Acc.} \\
\midrule
LLM & Single-turn & 0.9
& 72.40 & \(-0.40\) \\
LLM & Single-turn & 1.0
& 72.80 & \(+0.00\) \\
LLM & Single-turn & 1.1
& 73.00 & \(+0.20\) \\
LLM & Multi-t. private & 0.9
& 72.60 & \(+0.20\) \\
LLM & Multi-t. private & 1.0
& 72.40 & \(+0.00\) \\
LLM & Multi-t. private & 1.1
& 72.80 & \(+0.40\) \\
\midrule
QBAF & Single-turn & 1.0
& 74.40 & \(+0.00\) \\
QBAF & Single-turn & 1.1
& 74.40 & \(+0.00\) \\
QBAF & Single-turn & 1.2
& 74.20 & \(-0.20\) \\
QBAF & Multi-t. private & 1.0
& 75.40 & \(+0.00\) \\
QBAF & Multi-t. private & 1.1
& 75.40 & \(+0.00\) \\
QBAF & Multi-t. private & 1.2
& 75.80 & \(+0.40\) \\
\bottomrule
\end{tabular}
}
\end{table}

\noindent \textbf{Profile Robustness}
We measure profile robustness by examining the change in accuracy after perturbing the opinion profile, where all pro argument scores are increased by 0.1 and capped at 1.0. A smaller change in accuracy indicates greater profile robustness. As shown in Table~\ref{tab:profile-robustness}, the LLM judge \ad{accuracy changes by a magnitude of 2.40\% for the single-turn and 1.60\% for the multi-turn scenario}, whereas the QBAF judge changes by only \ad{0.20\% and 0.40\% in both scenarios, respectively}. \ad{This indicates} that QBAF-based judging is more profile-robust under the same perturbation.

\noindent \textbf{Judge Robustness}
We measure judge robustness by perturbing method-specific parameters and observing the resulting change in accuracy. For LLM-as-a-judge, we perturb the temperature by \(\pm 0.1\). For QBAF-based judging, we perturb the conservativeness parameter of the DF-QuAD semantics, which controls how strongly aggregated supporting and attacking arguments influence the final strength of the topic argument. As shown in Table~\ref{tab:judge-robustness}, both methods change only slightly by at most \(0.40\%\). This suggests that both methods are relatively stable within the tested parameter ranges. Since temperature and conservativeness affect different decision mechanisms, \ad{the results illustrate within-method robustness but are not directly comparable.}
\begin{table}[t]
\centering
\small
\caption{Non-hallucination results for the semantics-based
and LLM-based judges with prior stance scores.
\emph{Rate} denotes the proportion of judgements satisfying
non-hallucination.}
\label{tab:non-hallucination}
\begin{tabular}{lllcc}
\toprule
\textbf{Method} & \textbf{Scenario} & \textbf{Setting} & \textbf{Satisfied} & \textbf{Rate} \\
\midrule
LLM & Single-turn & No anon. & 471/500 & 94.20 \\
LLM & Single-turn & Anon. & 471/500 & 94.20 \\
LLM & Multi-t. private & No anon. & 475/500 & 95.00 \\
LLM & Multi-t. private & Anon. & 480/500 & 96.00 \\
LLM & Multi-t. shared & No anon. & 475/500 & 95.00 \\
LLM & Multi-t. shared & Anon. & 478/500 & 95.60 \\
\midrule
QBAF & Single-turn & Separate & 489/500 & 97.80 \\
QBAF & Single-turn & Combined & 494/500 & 98.80 \\
QBAF & Multi-t. private & Separate & 494/500 & 98.80 \\
QBAF & Multi-t. private & Combined & \textbf{498/500} & \textbf{99.60} \\
QBAF & Multi-t. shared & Separate & \textbf{498/500} & \textbf{99.60} \\
QBAF & Multi-t. shared & Combined & \textbf{498/500} & \textbf{99.60} \\
\bottomrule
\end{tabular}
\end{table}

% Judge-unanimity results (round 2).
% Agent positions taken from the agents' prior stance scores.
\begin{table}[t]
\centering
\small
\caption{Judge-unanimity results for the semantics-based
and LLM-based judges. \emph{Rate} denotes the proportion of judgements satisfying judge unanimity for unanimous opinions.}
\label{tab:judge-unanimity}
\begin{tabular}{lllcc}
\toprule
\textbf{Method} & \textbf{Scenario} & \textbf{Setting} & \textbf{Satisfied} & \textbf{Rate} \\
\midrule
LLM & Single-turn & No anon. & 287/316 & 90.82 \\
LLM & Single-turn & Anon. & 287/316 & 90.82 \\
LLM & Multi-t. private & No anon. & 291/316 & 92.09 \\
LLM & Multi-t. private & Anon. & 296/316 & 93.67 \\
LLM & Multi-t. shared & No anon. & 301/326 & 92.33 \\
LLM & Multi-t. shared & Anon. & 304/326 & 93.25 \\
\midrule
QBAF & Single-turn & Separate & 305/316 & 96.52 \\
QBAF & Single-turn & Combined & 310/316 & 98.10 \\
QBAF & Multi-t. private & Separate & 310/316 & 98.10 \\
QBAF & Multi-t. private & Combined & 314/316 & 99.37 \\
QBAF & Multi-t. shared & Separate & \textbf{324/326} & \textbf{99.39} \\
QBAF & Multi-t. shared & Combined & \textbf{324/326} & \textbf{99.39} \\
\bottomrule
\end{tabular}
\end{table}

\noindent \textbf{Non-hallucination}
Table~\ref{tab:non-hallucination} reports the non-hallucination results, where a judgement satisfies the property if its final decision matches the stance of at least one agent. The QBAF-based methods universally achieve higher satisfaction rates, \ad{with the best rate of 99.60\%, compared with 96\% for the best LLM-based variant}. This suggests that QBAF-based judging is less likely to produce a final decision that is unsupported by any agent stance.

% OLD TABLE, MERGED INTO THE TABLE ABOVE
% \begin{table}[t]
% \centering
% \small
% \caption{Non-hallucination results for the QBAF-semantics-as-a-judge
% and LLM-as-a-judge methods with prior stance scores.
% \emph{Rate} denotes the proportion of judgements satisfying
% non-hallucination. All rates are reported as percentages.
% Best results are shown in bold.}
% \label{tab:non-hallucination}
% \begin{tabular}{lcc}
% \toprule
% \textbf{Method} & \textbf{Satisfied} & \textbf{Rate} \\
% \midrule
% LLM No anon. & 471/500 & 94.20 \\
% LLM Anon. & 471/500 & 94.20 \\
% \midrule
% QBAF Combined & \textbf{494/500} & \textbf{98.80} \\
% QBAF Separate & 489/500 & 97.80 \\
% \bottomrule
% \end{tabular}
% \end{table}

\noindent \textbf{Judge Unanimity}
Table~\ref{tab:judge-unanimity} reports the judge-unanimity results. This property is evaluated on the opinion profiles where all three agents have the same stance. The QBAF-based methods achieve higher satisfaction rates, with \ad{the best rate of 99.39\% compared to 93.67\% for LLMs}. This suggests that QBAF-based judging better preserves unanimous agent agreement in this experimental setting.

% OLD TABLE MERGED INTO A COMBINED TABLE ABOVE
% \begin{table}[t]
% \centering
% \small
% \caption{Judge-unanimity results for the QBAF-semantics-as-a-judge
% and LLM-as-a-judge methods with prior stance scores. The property is
% evaluated only on the \(316\) unanimous opinion profiles.
% \emph{Rate} denotes the proportion of judgements satisfying judge
% unanimity. All rates are reported as percentages.
% Best results are shown in bold.}
% \label{tab:judge-unanimity}
% \begin{tabular}{lcc}
% \toprule
% \textbf{Method} & \textbf{Satisfied} & \textbf{Rate} \\
% \midrule
% LLM No anon. & 287/316 & 90.82 \\
% LLM Anon. & 287/316 & 90.82 \\
% \midrule
% QBAF Combined & \textbf{310/316} & \textbf{98.10} \\
% QBAF Separate & 305/316 & 96.52 \\
% \bottomrule
% \end{tabular}
% \end{table}

\noindent \textbf{Contestability}
Using the same perturbation setting as the profile-robustness experiment, we increase all pro argument base scores by \(0.1\), capped at \(1.0\), and evaluate whether the judge output moves in the expected positive direction. As shown in Table~\ref{tab:contestability},
%\ft{the SM}
the semantics-as-a-judge satisfies contestability in all 500 cases. This follows from the \emph{monotonicity} of DF-QuAD semantics: increasing the base scores of pro arguments cannot decrease the final strength of the topic argument \cite{baroni2018many}. In contrast, the LLM-based judge does not strictly satisfy contestability: \ad{up to 93/500 scores decrease and up to 6 predictions} change from True to False after strengthening the pro arguments. This indicates that LLM-as-a-judge is not constrained by a monotonic update rule, so increasing pro evidence may still lead to a less supportive final judgement.
\begin{table}[t]
\centering
\small
\caption{Contestability results for increasing all pro scores by $0.1$.
\emph{Score $\geq$} reports cases where the judge score after perturbation is greater than or equal to the original score.}
\label{tab:contestability}
\begin{tabular}{llccc}
\toprule
\textbf{Method} & \textbf{Scenario} & \textbf{Score $\geq$} & \textbf{F \(\rightarrow\) T} & \textbf{T \(\rightarrow\) F} \\
\midrule
LLM & Single-turn & 407/500 & 14 & 6 \\
LLM & Multi-t. private & 442/500 & 14 & 2 \\
\midrule
QBAF & Single-turn & \textbf{500/500} & 15 & \textbf{0} \\
QBAF & Multi-t. private & \textbf{500/500} & 6 & \textbf{0} \\
\bottomrule
\end{tabular}
\end{table}

\section{Conclusions}
Overall, we make the following contributions: 1) We define a novel set of formal properties for post-hoc debate judgement in the abstract;
    2) We instantiate the resulting theory to the setting of claim verification, in a number of different scenarios depending on whether agents conduct debates internally or directly interact% with one another
    ; 3) We explore the satisfaction of the proposed properties %formally and
    experimentally, when the judges are either LLMs or argumentation semantics.
%In summary,
Our study indicates argumentation semantics as an ideal candidate for principled judges in debate-driven AI.
Several directions for future work remain open though.
First, our list of properties is not exhaustive, e.g. it would be interesting to formalise and evaluate properties of 
open-mindedness, requiring that judges do not focus only on a subset of the agents' opinions%, %every opinion matter...
%and judge faithfulness to the agents' opinions
. 
Second, it would be interesting to explore  a larger number of agents, rounds, different LLMs as agents and as judges, and different argumentation semantics.
Third, 
we focused on experimental ananlysis of properties, but it would be interesting to conduct a theoretical analysis too. 
%Fourthly, %we restricted attention to shallow explanations only.
Finally, we focused on claim verification, but other downstream tasks, e.g. questions answering, would be worthwhile. 
%``round-based'' debates may satisfy some properties within each round, but not over the whole debate? e.g. permutation-independence .... was round-based \cite{debateLLMs-2024}

%\todo{@XY: double-check the whole paper: british english}

\newpage
\bibliography{aaai2027}

% \section*{Acknowledgments}

\newpage
\setcounter{page}{1}
\onecolumn
\appendix
\section*{Supplementary Material for\\``A Theory of Post-hoc Debate Judgement''}

\section{Computing Environment}
\paragraph{Single-Turn Experiments.}
All experiments were conducted locally on a personal laptop running Microsoft Windows 11 Home, equipped with an Intel Core Ultra 5 225H CPU (14 cores) and 32 GB of RAM. All computations were performed on the CPU or by remote providers accessed via an API. The software stack consisted of Python 3.11.9, OpenAI 1.40.6, NumPy 2.4.0, pandas 2.3.3, scikit-learn 1.8.0, and Matplotlib 3.10.8.

\paragraph{Multi-Turn Experiments.} All experiments were conducted on a desktop computer running a customized distribution of Ubuntu 22.04.5 LTS, equipped with an AMD® Ryzen 7 pro 3700 8-core CPU and 16 GB of RAM. All computations were performed on the CPU or by remote providers accessed via an API. The software stack consisted of Python 3.12.10, OpenAI 2.44.0, Pydantic 2.13.4 and Tenacity 9.1.4.

\section{Experimental Results}
\label{sec_expr_results_appendix}
\subsection{Additional Results for Determinism and Permutation Independence}

\begin{table}[H]
\centering
\small
\caption{Determinism results for the LLM-as-a-judge method under different temperature settings.}
\label{tab:llm-determinism}
\begin{tabular}{lccccc}
\toprule
\textbf{Scenario} & \textbf{Temp.} & \textbf{Run}
& \textbf{Acc.} & \textbf{T-Acc.} & \textbf{F-Acc.} \\
\midrule
Single-turn & 1.0 & 1 & 72.40 & 52.21 & 92.43 \\
Single-turn & 1.0 & 2 & 72.60 & 52.21 & 92.83 \\
Single-turn & 1.0 & 3 & 73.60 & 54.62 & 92.43 \\
\midrule
Single-turn & 0.0 & 1 & 73.00 & 53.01 & 92.83 \\
Single-turn & 0.0 & 2 & 72.60 & 52.21 & 92.83 \\
Single-turn & 0.0 & 3 & 73.20 & 53.41 & 92.83 \\
\midrule
Multi-t. private & 1.0 & 1 & 72.80 & 51.81 & 93.63 \\
Multi-t. private & 1.0 & 2 & 73.00 & 51.81 & 94.02 \\
Multi-t. private & 1.0 & 3 & 72.20 & 50.20 & 94.02 \\
\midrule
Multi-t. private & 0.0 & 1 & 73.20 & 52.21 & 94.02 \\
Multi-t. private & 0.0 & 2 & 73.40 & 52.61 & 94.02 \\
Multi-t. private & 0.0 & 3 & 73.20 & 52.61 & 93.63 \\
\bottomrule
\end{tabular}
\end{table}

\begin{table}[H]
\centering
\small
\caption{Permutation-independence results for the LLM-as-a-judge method under different agent orders.}
\label{tab:llm-permutation}
\begin{tabular}{llccc}
\toprule
\textbf{Scenario} & \textbf{Agent Order}
& \textbf{Acc.} & \textbf{T-Acc.} & \textbf{F-Acc.} \\
\midrule
Single-turn & A1--A2--A3 & 72.80 & 53.41 & 92.03 \\
Single-turn & A1--A3--A2 & 73.80 & 55.02 & 92.43 \\
Single-turn & A2--A1--A3 & 72.80 & 53.01 & 92.43 \\
Single-turn & A2--A3--A1 & 73.40 & 54.62 & 92.03 \\
Single-turn & A3--A1--A2 & 73.60 & 54.22 & 92.83 \\
Single-turn & A3--A2--A1 & 74.40 & 56.22 & 92.43 \\
\midrule
Multi-t. private & A1--A2--A3 & 73.20 & 51.81 & 94.42 \\
Multi-t. private & A1--A3--A2 & 73.20 & 51.81 & 94.42 \\
Multi-t. private & A2--A1--A3 & 73.40 & 52.61 & 94.02 \\
Multi-t. private & A2--A3--A1 & 73.20 & 52.21 & 94.02 \\
Multi-t. private & A3--A1--A2 & 73.40 & 53.41 & 93.23 \\
Multi-t. private & A3--A2--A1 & 72.60 & 52.61 & 92.43 \\
\bottomrule
\end{tabular}
\end{table}

\subsection{Additional Results for QE Semantics}
\paragraph{Quadratic Energy Semantics}
Quadratic Energy Model (QE)~\cite{Potyka18}, like many other QBAF semantics, determines the strength of an argument from its base score and the strengths of its direct supporters and attackers. For an acyclic QBAF, the final strengths can be computed in a single pass by evaluating the arguments in a topological order, beginning with arguments that have no incoming relations.
Let
\(\sigma \colon \mathcal{A} \rightarrow [0,1]\)
denote the resulting strength assignment. For every argument \(\alpha \in \mathcal{A}\), QE first computes its \emph{energy} by subtracting the total strength of its attackers from the total strength of its supporters:
\(
E_\alpha
=
\sum_{\{\beta \in \mathcal{A} \mid (\beta,\alpha) \in \mathcal{R}^{+}\}}
\sigma(\beta)
-
\sum_{\{\beta \in \mathcal{A} \mid (\beta,\alpha) \in \mathcal{R}^{-}\}}
\sigma(\beta).
\)
Since the arguments are evaluated in topological order, the strengths of all direct supporters and attackers of \(\alpha\) are already available when \(E_\alpha\) is computed.
If \(E_\alpha \leq 0\), the final dialectical strength of \(\alpha\) is
\(
\sigma(\alpha)
=
\tau(\alpha)
-
\tau(\alpha)\cdot
\frac{E_\alpha^2}{1+E_\alpha^2}.
\)
If \(E_\alpha > 0\), its final dialectical strength is
\(
\sigma(\alpha)
=
\tau(\alpha)
+
(1-\tau(\alpha))\cdot
\frac{E_\alpha^2}{1+E_\alpha^2}.
\)
Thus, a non-positive energy decreases the strength of the argument relative to its base score, whereas a positive energy increases it. For an argument with no supporters or attackers, \(E_\alpha=0\), and hence \(\sigma(\alpha)=\tau(\alpha)\).

\begin{table}[H]
\centering
\small
\caption{Accuracy results for QBAF-based judging.}
\label{tab:judgement-accuracy-qe}
{\setlength{\tabcolsep}{3pt}
\begin{tabular}{lllccc}
\toprule
\textbf{Scenario} & \textbf{Setting} & \textbf{Variant} 
& \textbf{Acc.} & \textbf{T-Acc.} & \textbf{F-Acc.} \\
\midrule
\multicolumn{6}{l}{\textit{QE Semantics-as-a-judge}} \\
\midrule
Single-turn & Separate & 0.5        & 68.40 & 48.59 & 88.05 \\
Single-turn & Separate & Prior      & \textbf{75.20} & \textbf{63.86} & 86.45 \\
Single-turn & Combined & 0.5        & 68.40 & 48.59 & 88.05 \\
Single-turn & Combined & Avg. Prior & 73.40 & 56.63 & 90.04 \\
\midrule
Private & Separate & 0.5 & 72.20 & 50.60 & 93.63 \\
Private & Separate & Prior & 75.00 & 57.83 & 92.03 \\
Private & Combined & 0.5 & 72.00 & 50.20 & 93.63 \\
Private & Combined & Prior & 72.20 & 52.21 & 92.03 \\
Shared & Separate & 0.5 & 68.60 & 43.37 & 93.63 \\
Shared & Separate & Prior & 69.80 & 45.38 & \textbf{94.02} \\
Shared & Combined & 0.5 & 68.60 & 43.37 & 93.63 \\
Shared & Combined & Prior & 69.80 & 45.38 & \textbf{94.02} \\
\bottomrule
\end{tabular}
}
\end{table}

\begin{table}[H]
\centering
\small
\caption{Profile-robustness results for QBAF-based judging.}
\label{tab:qe-profile-robustness}
\begin{tabular}{lllcc}
\toprule
\textbf{Method} & \textbf{Scenario} & \textbf{Profile}
& \textbf{Acc.} & \(\boldsymbol{\Delta}\)\textbf{ Acc.} \\
\midrule
QBAF (DF-QuAD) & Single-turn & Baseline
& 74.40 & \(+0.00\) \\
QBAF (DF-QuAD) & Single-turn & Pro \(+0.1\), capped
& 74.60 & \(+0.20\) \\
QBAF (DF-QuAD) & Multi-t. private & Baseline & 75.40 & \(+0.00\) \\
QBAF (DF-QuAD) & Multi-t. private & Pro \(+0.1\), capped & 75.00 & \(-0.40\) \\
\midrule
QBAF (QE) & Single-turn & Baseline
& 75.20 & \(+0.00\) \\
QBAF (QE) & Single-turn & Pro \(+0.1\), capped
& 74.80 & \(-0.40\) \\
QBAF (QE) & Multi-t. private & Baseline & 75.00 & \(+0.00\) \\
QBAF (QE) & Multi-t. private & Pro \(+0.1\), capped & 76.00 & \(+1.00\) \\
\bottomrule
\end{tabular}
\end{table}

\begin{table}[H]
\centering
\small
\caption{Judge-robustness results for QBAF-based judging under different conservativeness settings.}
\label{tab:qe-judge-robustness}
\begin{tabular}{lllcc}
\toprule
\textbf{Method} & \textbf{Scenario} & \textbf{Conserv.}
& \textbf{Acc.} & \(\boldsymbol{\Delta}\)\textbf{ Acc.} \\
\midrule
QBAF (DF-QuAD) & Single-turn & 1.0
& 74.40 & \(+0.00\) \\
QBAF (DF-QuAD) & Single-turn & 1.1
& 74.40 & \(+0.00\) \\
QBAF (DF-QuAD) & Single-turn & 1.2
& 74.20 & \(-0.20\) \\
QBAF (DF-QuAD) & Multi-t. private & 1.0
& 75.40 & \(+0.00\) \\
QBAF (DF-QuAD) & Multi-t. private & 1.1
& 75.40 & \(+0.00\) \\
QBAF (DF-QuAD) & Multi-t. private & 1.2
& 75.80 & \(+0.40\) \\
\midrule
QBAF (QE) & Single-turn & 1.0
& 75.20 & \(+0.00\) \\
QBAF (QE) & Single-turn & 1.1
& 75.20 & \(+0.00\) \\
QBAF (QE) & Single-turn & 1.2
& 75.00 & \(-0.20\) \\
QBAF (QE) & Multi-t. private & 1.0 & 75.00 & \(+0.00\) \\
QBAF (QE) & Multi-t. private & 1.1 & 74.80 & \(-0.20\) \\
QBAF (QE) & Multi-t. private & 1.2 & 74.80 & \(-0.20\) \\
\bottomrule
\end{tabular}
\end{table}

\begin{table}[H]
\centering
\small
\caption{Non-hallucination results for QBAF-based judging.}
\label{tab:qe-non-hallucination}
\begin{tabular}{llccc}
\toprule
\textbf{Method} & \textbf{Scenario} & \textbf{Setting}
& \textbf{Satisfied} & \textbf{Rate} \\
\midrule
QBAF (DF-QuAD) & Single-turn & Separate
& 489/500 & 97.80 \\
QBAF (DF-QuAD) & Single-turn & Combined
& 494/500 & 98.80 \\
QBAF (DF-QuAD) & Multi-t. private & Separate & 494/500 & 98.80 \\
QBAF (DF-QuAD) & Multi-t. private & Combined & 498/500 & 99.60 \\
QBAF (DF-QuAD) & Multi-t. shared & Separate & 498/500 & 99.60 \\
QBAF (DF-QuAD) & Multi-t. shared & Combined & 498/500 & 99.60 \\
\midrule
QBAF (QE) & Single-turn & Separate
& 497/500 & 99.40 \\
QBAF (QE) & Single-turn & Combined
& 483/500 & 96.60 \\
QBAF (QE) & Multi-t. private  & Separate & 482/500 & 96.40 \\
QBAF (QE) & Multi-t. private  & Combined & 473/500 & 94.60 \\
QBAF (QE) & Multi-t. private & Separate & 462/500 & 92.40 \\
QBAF (QE) & Multi-t. private & Combined & 462/500 & 92.40 \\
\bottomrule
\end{tabular}
\end{table}

\begin{table}[H]
\centering
\small
\caption{Judge-unanimity results for QBAF-based judging.}
\label{tab:qe-judge-unanimity}
\begin{tabular}{llccc}
\toprule
\textbf{Method} & \textbf{Scenario} & \textbf{Setting}
& \textbf{Satisfied} & \textbf{Rate} \\
\midrule
QBAF (DF-QuAD) & Single-turn & Separate
& 305/316 & 96.52 \\
QBAF (DF-QuAD) & Single-turn & Combined
& 310/316 & 98.10 \\
QBAF (DF-QuAD) & Multi-t. private & Separate & 310/316 & 98.10 \\
QBAF (DF-QuAD) & Multi-t. private & Combined & 314/316 & 99.37 \\
QBAF (DF-QuAD) & Multi-t. shared & Separate & 324/326 & 99.39 \\
QBAF (DF-QuAD) & Multi-t. shared & Combined & 324/326 & 99.39 \\
\midrule
QBAF (QE) & Single-turn & Separate
& 313/316 & 99.05 \\
QBAF (QE) & Single-turn & Combined
& 299/316 & 94.62 \\
QBAF (QE) & Multi-t. private & Separate & 298/316 & 94.30 \\
QBAF (QE) & Multi-t. private & Combined & 289/316 & 91.46 \\
QBAF (QE) & Multi-t. shared & Separate & 288/326 & 88.34 \\
QBAF (QE) & Multi-t. shared & Combined & 288/326 & 88.34 \\
\bottomrule
\end{tabular}
\end{table}

\begin{table}[h]
\centering
\small
\caption{Contestability results for QBAF-based judging.}
\label{tab:qe-contestability}
\begin{tabular}{llccc}
\toprule
\textbf{Method} & \textbf{Scenario}
& \textbf{Score Non-decreasing} & \textbf{False to True} & \textbf{True to False} \\
\midrule
QBAF (DF-QuAD) & Single-turn
& 500/500 & 15 & 0 \\
QBAF & Multi-t. private & 500/500 & 6 & 0 \\
\midrule
QBAF (QE) & Single-turn
& 500/500 & 6 & 0 \\
QBAF (QE) & Multi-t. private & 500/500 & 17 & 0 \\
\bottomrule
\end{tabular}
\end{table}

\newpage
\section{Proofs}
\setcounter{property}{0}
In the following, we prove that the QBAF-based judges used in this paper satisfy
determinism and permutation independence. For the remaining properties, we
provide empirical evaluations and compare QBAF-based judges with LLM-as-a-judge.

\begin{property}[Determinism]
A method \(c\in \mathbb{C}\) satisfies \emph{determinism} iff, for every  opinion profile 
\(\mathcal{O}=(O_1,\ldots,O_n)\in\mathbb{O}^n\), there exists a unique \(y^*\) such that
\(
c(I,\mathcal{O})=(y^*,e^*)
\)
for some explanation \(e^*\).
\end{property}
\begin{proposition}
The QBAF-based judges using DF-QuAD or QE gradual semantics satisfy \emph{determinism}.
\end{proposition}

\begin{proof}
For any fixed opinion profile \(\mathcal{O}\), the QBAF construction gives a
fixed set of arguments, relations, and base scores. Both DF-QuAD and QE compute
argument strengths using fixed aggregation and influence functions, so all
arguments have a unique strength. Hence, the QBAF-based judges using DF-QuAD or
QE semantics satisfy determinism.
\end{proof}

\begin{property}[Permutation Independence]
A method \(c\in \mathbb{C}\) satisfies \emph{permutation independence} iff, for every opinion profile \(\mathcal{O}\!=\!(O_1,\ldots,O_n)\in\mathbb{O}^n\)  and every permutation \(\mathcal{O}'\) of \(\mathcal{O}\),
if $c(I,\mathcal{O}) = (y^*, e^*)$ and $c(I,\mathcal{O}^{\prime}) = (y^{\prime}, e^{\prime})$, then $y^*=y^{\prime}$. 
\end{property}
\begin{proposition}
The QBAF-based judges using DF-QuAD or QE gradual semantics satisfy \emph{permutation independence}.
\end{proposition}
\begin{proof}
For any opinion profile \(\mathcal{O}\), permuting the order of opinions does not
change the constructed QBAF: the same arguments, support/attack relations, and
base scores are obtained. Since DF-QuAD and QE compute argument strengths from
this fixed QBAF using deterministic aggregation and influence functions, arguments receive the same strength under \(\mathcal{O}\) and its
permutation \(\mathcal{O}'\). Hence, the QBAF-based judges using
DF-QuAD or QE satisfy permutation independence.
\end{proof}

\section{Prompts}

\subsection{Single-Turn Data Generation Prompts}

\paragraph{Prompt 1: Prior Stance Score} \phantom{}

\begin{lstlisting}
System:
You are an independent debate agent. Return only valid JSON.

User:
You are one LLM debate agent.

Claim:
{claim}

Task:
Assign your final stance score toward the truthfulness of the claim.

Definitions:
- 0.0 means you think the claim is very likely false.
- 0.5 means you think the evidence is balanced or uncertain.
- 1.0 means you think the claim is very likely true.

Important constraints:
- You do not know the gold label.
- Do not generate pro/con arguments.
- Do not generate base scores.
- Do not output a support/oppose label; the script will derive it from the score.
- Output the score with exactly two decimal places.
- Return only valid JSON with exactly one key named "stance_score".

Return exactly this JSON structure:
{"stance_score": SCORE}

Replace SCORE with your own calibrated two-decimal number from 0.00 to 1.00.
\end{lstlisting}

\paragraph{Prompt 2: Pro/Con Argument Generation} \phantom{}
\begin{lstlisting}
System:
You generate compact argumentation data. Return only valid JSON.

User:
Please provide a single short argument {supporting/attacking} the following claim.

Construct the argument so it refers to the truthfulness of the claim. You must always return exactly one argument, even if the argument is weak. If the claim is very likely false and you are asked for a supporting argument, provide the best weak steelman reason someone might mistakenly think the claim is true. If the claim is very likely true and you are asked for an attacking argument, provide the best weak objection someone might raise. The argument should be a single short sentence.

{The argument must make the claim more likely to be true. / The argument must make the claim more likely to be false.}
{Do not say "no evidence", "not supported", "false", "baseless", or similar refutations in a supporting argument. / Do not provide a reason that makes the claim more likely to be true.}

Provide your argument in the following JSON form:
[
  "Argument text"
]

Claim: {claim}

Now come up with the argument.
\end{lstlisting}

\paragraph{Prompt 3: Argument Base Score} \phantom{}
\begin{lstlisting}
System:
You evaluate argument validity and relevance. Return only the requested likelihood line.

User:
You are an analyst evaluating the validity and relevance of arguments.

For the argument:

Argument: "{argument}"

please give your confidence that the argument presents a compelling case {in favour of / against} the statement:

Statement: "{claim}"

Your assessment should be based on how well the argument {supports / refutes} the considered statement as well as the correctness, accuracy and truthfulness of the given argument.

Your response should be between 0% and 100%, with 0% indicating that the considered argument is definitely invalid, 100% indicating that the considered argument is definitely valid, and values in between indicating various levels of uncertainty. Your estimates should be well-calibrated, so feel free to err on the side of caution and output moderate probabilities if you are not completely sure in your assessment.

Please respond in the following form:

Likelihood: The predicted likelihood that the considered argument is valid.
\end{lstlisting}

\paragraph{Prompt 4: Posterior Stance Score} \phantom{}
\begin{lstlisting}
System:
You are an independent debate agent. Return only valid JSON.

User:
You are one LLM debate agent.

Claim:
{claim}

You have already generated one pro argument, one con argument, and base scores
for both arguments. You may use the current pro/con content and their base
scores when assigning your final posterior stance score.

Pro argument:
{pro_argument}

Pro base score:
{pro_score}

Con argument:
{con_argument}

Con base score:
{con_score}

Task:
Assign your posterior stance score toward the truthfulness of the claim.

Definitions:
- 0.0 means you think the claim is very likely false.
- 0.5 means you think the evidence is balanced or uncertain.
- 1.0 means you think the claim is very likely true.

Important constraints:
- You do not know the gold label.
- Do not generate new pro/con arguments.
- Do not generate new base scores.
- Do not output a support/oppose label; the script will derive it from the score.
- Output the score with exactly two decimal places.
- Return only valid JSON with exactly one key named "posterior_stance_score".

Return exactly this JSON structure:
{"posterior_stance_score": SCORE}

Replace SCORE with your calibrated two-decimal number from 0.00 to 1.00.
\end{lstlisting}

\subsection{Single-Turn Judge Prompts}

\paragraph{Prompt 1: LLM Judge (Named Agents)} \phantom{}

\begin{lstlisting}
System:
You are a careful judge for claim verification. Return only valid JSON.

User:
You are an LLM judge for a claim verification experiment.

Your task is to decide whether the claim is more likely true or false using only
the evidence below. You do not know the gold label.

Claim:
{claim}

Agent evidence:
Agent A1:
{A1 stance line, optional}
- pro argument: {A1 pro argument}
- pro base score: {A1 pro base score}
- con argument: {A1 con argument}
- con base score: {A1 con base score}

Agent A2:
{A2 stance line, optional}
- pro argument: {A2 pro argument}
- pro base score: {A2 pro base score}
- con argument: {A2 con argument}
- con base score: {A2 con base score}

Agent A3:
{A3 stance line, optional}
- pro argument: {A3 pro argument}
- pro base score: {A3 pro base score}
- con argument: {A3 con argument}
- con base score: {A3 con base score}

Instructions:
- {stance instruction}
- Agent identities and each agent's pro/con grouping are visible.
- Consider both the pro and con arguments and their base scores.
- The base scores are self-assessed by the agent that generated each argument.
- Assign a truth score from 0.0 to 1.0.
- A score of 0.0 means the claim is very likely false.
- A score of 0.5 means the evidence is balanced or uncertain.
- A score of 1.0 means the claim is very likely true.
- Do not output a true/false prediction; the script will convert score >= 0.5
  to true and score < 0.5 to false.
- Return only valid JSON.

Required JSON:
{
  "score": 0.0
}
\end{lstlisting}

\paragraph{Prompt 2: Judge Evidence Block (Anonymised)} \phantom{}

\begin{lstlisting}
The following evidence has been anonymised. Agent identifiers are hidden,
and the order of evidence items should not be interpreted as indicating
which items were generated by the same agent.

Anonymous stance information:
{stance lines, optional}

Support arguments:
- argument: {pro argument 1}
- base score: {pro base score 1}
...
- argument: {pro argument 3}
- base score: {pro base score 3}

Attack arguments:
- argument: {con argument 1}
- base score: {con base score 1}
...
- argument: {con argument 3}
- base score: {con base score 3}
\end{lstlisting}

\paragraph{Prompt 3: Judge Instruction Line Variants} \phantom{}

\begin{lstlisting}
Stance note:

[stance = no]
- The agents' support/oppose predictions and stance scores are intentionally hidden.

[stance = prior; named agents]
- You are allowed to use the agents' stance scores as evidence; support/oppose
  predictions are intentionally hidden.

[stance = prior; anonymised]
- You are allowed to use the anonymous stance scores as evidence; support/oppose
  predictions are intentionally hidden.

[stance = posterior; named agents]
- You are allowed to use the agents' posterior stance scores as evidence; support/oppose
  predictions are intentionally hidden.

[stance = posterior; anonymised]
- You are allowed to use the anonymous posterior stance scores as evidence; support/oppose
  predictions are intentionally hidden.

Identity note:

[named agents]
- Agent identities and each agent's pro/con grouping are visible.

[anonymised]
- Agent identities and the original agent-level grouping of stance scores,
  pro arguments, and con arguments are hidden.
\end{lstlisting}

\subsection{Multi-Turn Data Generation Prompts}

\paragraph{Prompt 1: Prior Stance Score (Both Scenarios, Round 0)} \phantom{}

\begin{lstlisting}
System:
You are an independent debate agent. Return only valid JSON.

User:
You are one LLM debate agent.

Claim:
{claim}

Task:
Assign your stance score toward the truthfulness of the claim.

Definitions:
- 0.0 means you think the claim is very likely false.
- 0.5 means you think the evidence is balanced or uncertain.
- 1.0 means you think the claim is very likely true.

Important constraints:
- You do not know the gold label.
- Do not generate pro/con arguments.
- Do not generate base scores.
- Do not output a support/oppose label; the script will derive it from the score.
- Output the score with exactly two decimal places.
- Return only valid JSON with exactly one key named "stance_score".

Return exactly this JSON structure:
{"stance_score": SCORE}

Replace SCORE with your own calibrated two-decimal number from 0.00 to 1.00.
\end{lstlisting}

\paragraph{Prompt 2: Argument Generation (Both Scenarios, Round 0)} \phantom{}

\begin{lstlisting}
System:
You generate compact argumentation data. Return only valid JSON.

User:
Please construct short arguments in favour and against the following claim.

Construct every argument so that it refers to the truthfulness of the claim. You may provide any number of pro and con arguments, but try to provide at least one of each, even if an argument is weak. If the claim is very likely false, still provide the best weak steelman reasons someone might mistakenly think the claim is true as pro arguments. If the claim is very likely true, still provide the best weak objections someone might raise as con arguments. Each argument must be a single short sentence.

Each pro argument must make the claim more likely to be true, and each con argument must make the claim more likely to be false.

Provide your arguments in the following JSON form:
{
  "pros": ["Pro argument text", ...],
  "cons": ["Con argument text", ...]
}

Claim:
{claim}

Now come up with the arguments.
\end{lstlisting}

\paragraph{Prompt 3: Argument Base Score (Both Scenarios)} \phantom{}

\begin{lstlisting}
System:
You evaluate argument validity and relevance. Return only the requested likelihood line.

User:
You are an analyst evaluating the validity and relevance of arguments.

For the argument:

Argument: "{argument}"

please give your confidence that the argument presents a compelling case {in favour of / against} the statement:

Statement: "{claim}"

Your assessment should be based on how well the argument {supports / refutes} the considered statement as well as the correctness, accuracy and truthfulness of the given argument.

Your response should be between 0% and 100%, with 0% indicating that the considered argument is definitely invalid, 100% indicating that the considered argument is definitely valid, and values in between indicating various levels of uncertainty. Your estimates should be well-calibrated, so feel free to err on the side of caution and output moderate probabilities if you are not completely sure in your assessment.

Please respond in the following form:

Likelihood: The predicted likelihood that the considered argument is valid.
\end{lstlisting}

\paragraph{Prompt 4: Posterior Stance Score (Private Scenario, Round 0)} \phantom{}

\begin{lstlisting}
System:
You are an independent debate agent. Return only valid JSON.

User:
You are one LLM debate agent.

Claim:
{claim}

Your own arguments about this claim:
Pro arguments (supporting the truthfulness of the claim):
  - "{pro argument 1}" (base score {pro base score 1})
  ...
  - "{pro argument P}" (base score {pro base score P})
Con arguments (attacking the truthfulness of the claim):
  - "{con argument 1}" (base score {con base score 1})
  ...
  - "{con argument C}" (base score {con base score C})

Argument base scores indicate the confidence of the agent that added each argument in the validity of that argument.

Task:
Taking the information above into account, assign your stance score toward the truthfulness of the claim.

Definitions:
- 0.0 means you think the claim is very likely false.
- 0.5 means you think the evidence is balanced or uncertain.
- 1.0 means you think the claim is very likely true.

Important constraints:
- You do not know the gold label.
- Do not generate pro/con arguments.
- Do not generate base scores.
- Do not output a support/oppose label; the script will derive it from the score.
- Output the score with exactly two decimal places.
- Return only valid JSON with exactly one key named "stance_score".

Return exactly this JSON structure:
{"stance_score": SCORE}

Replace SCORE with your own calibrated two-decimal number from 0.00 to 1.00.
\end{lstlisting}

\paragraph{Prompt 5: Argument Refinement (Private Scenario, Round 1 Onwards)} \phantom{}

\begin{lstlisting}
System:
You generate compact argumentation data. Return only valid JSON.

User:
You are an LLM debate agent refining your arguments about the truthfulness of the following claim after a round of debate.

Claim:
{claim}

Your previous stance score toward this claim was {your previous stance score} (0.00 = very likely false, 1.00 = very likely true). You previously put forward the following arguments:
Pro arguments (supporting the truthfulness of the claim):
  - "{pro argument 1}" (base score {pro base score 1})
  ...
  - "{pro argument P}" (base score {pro base score P})
Con arguments (attacking the truthfulness of the claim):
  - "{con argument 1}" (base score {con base score 1})
  ...
  - "{con argument C}" (base score {con base score C})

The other agents put forward the following arguments and stances:
Agent A2 (stance score {A2 stance score}):
  Pro arguments (supporting the truthfulness of the claim):
    - "{A2 pro argument 1}" (base score {A2 pro base score 1})
    ...
    - "{A2 pro argument P}" (base score {A2 pro base score P})
  Con arguments (attacking the truthfulness of the claim):
    - "{A2 con argument 1}" (base score {A2 con base score 1})
    ...
    - "{A2 con argument C}" (base score {A2 con base score C})

Agent A3 (stance score {A3 stance score}):
  Pro arguments (supporting the truthfulness of the claim):
    - "{A3 pro argument 1}" (base score {A3 pro base score 1})
    ...
    - "{A3 pro argument P}" (base score {A3 pro base score P})
  Con arguments (attacking the truthfulness of the claim):
    - "{A3 con argument 1}" (base score {A3 con base score 1})
    ...
    - "{A3 con argument C}" (base score {A3 con base score C})

Argument base scores indicate the confidence of the agent that added each argument in the validity of that argument.

Please provide a revised set of pro and con arguments about the truthfulness of the claim. You may keep, edit, remove, or add arguments in light of the debate, and you should engage with the strongest points raised by the other agents.

Construct every argument so that it refers to the truthfulness of the claim. You may provide any number of pro and con arguments, but try to provide at least one of each, even if an argument is weak. If the claim is very likely false, still provide the best weak steelman reasons someone might mistakenly think the claim is true as pro arguments. If the claim is very likely true, still provide the best weak objections someone might raise as con arguments. Each argument must be a single short sentence.

Each pro argument must make the claim more likely to be true, and each con argument must make the claim more likely to be false.

Provide your arguments in the following JSON form:
{
  "pros": ["Pro argument text", ...],
  "cons": ["Con argument text", ...]
}
\end{lstlisting}

\paragraph{Prompt 6: Posterior Stance Score (Private Scenario, Round 1 Onwards)} \phantom{}

\begin{lstlisting}
System:
You are an independent debate agent. Return only valid JSON.

User:
You are one LLM debate agent.

Claim:
{claim}

Your previous stance score toward this claim was {your previous stance score}.

Your own arguments about this claim:
Pro arguments (supporting the truthfulness of the claim):
  - "{pro argument 1}" (base score {pro base score 1})
  ...
  - "{pro argument P}" (base score {pro base score P})
Con arguments (attacking the truthfulness of the claim):
  - "{con argument 1}" (base score {con base score 1})
  ...
  - "{con argument C}" (base score {con base score C})

Arguments and stances put forward by the other agents:
Agent A2 (stance score {A2 stance score}):
  Pro arguments (supporting the truthfulness of the claim):
    - "{A2 pro argument 1}" (base score {A2 pro base score 1})
    ...
    - "{A2 pro argument P}" (base score {A2 pro base score P})
  Con arguments (attacking the truthfulness of the claim):
    - "{A2 con argument 1}" (base score {A2 con base score 1})
    ...
    - "{A2 con argument C}" (base score {A2 con base score C})

Agent A3 (stance score {A3 stance score}):
  Pro arguments (supporting the truthfulness of the claim):
    - "{A3 pro argument 1}" (base score {A3 pro base score 1})
    ...
    - "{A3 pro argument P}" (base score {A3 pro base score P})
  Con arguments (attacking the truthfulness of the claim):
    - "{A3 con argument 1}" (base score {A3 con base score 1})
    ...
    - "{A3 con argument C}" (base score {A3 con base score C})

Argument base scores indicate the confidence of the agent that added each argument in the validity of that argument.

Task:
Taking the information above into account, assign your stance score toward the truthfulness of the claim. You may keep or revise your previous view; be persuaded only to the extent the arguments warrant it.

Definitions:
- 0.0 means you think the claim is very likely false.
- 0.5 means you think the evidence is balanced or uncertain.
- 1.0 means you think the claim is very likely true.

Important constraints:
- You do not know the gold label.
- Do not generate pro/con arguments.
- Do not generate base scores.
- Do not output a support/oppose label; the script will derive it from the score.
- Output the score with exactly two decimal places.
- Return only valid JSON with exactly one key named "stance_score".

Return exactly this JSON structure:
{"stance_score": SCORE}

Replace SCORE with your own calibrated two-decimal number from 0.00 to 1.00.
\end{lstlisting}

\paragraph{Prompt 7: Shared Framework Edit (Shared Scenario)} \phantom{}

\begin{lstlisting}
System:
You generate compact argumentation data. Return only valid JSON.

User:
You are one LLM debate agent collaboratively building a shared argumentation framework about the truthfulness of a claim.

Claim:
{claim}

Current shared argumentation framework:
Pro arguments (supporting the truthfulness of the claim):
  [pro1] "{pro argument 1}" (base score {pro base score 1}) - added by you
  ...
  [proP] "{pro argument P}" (base score {pro base score P}) - added by another agent
Con arguments (attacking the truthfulness of the claim):
  [con1] "{con argument 1}" (base score {con base score 1}) - added by you
  ...
  [conC] "{con argument C}" (base score {con base score C}) - added by another agent

Most recent stance scores toward the truthfulness of the claim (0.00 = very likely false, 1.00 = very likely true):
you: {your stance score}; agent A2: {A2 stance score}; agent A3: {A3 stance score}

Task:
You may (a) edit arguments that you previously added and (b) add new pro and/or con arguments. You do not have to make any change if you are satisfied with the framework.

Constraints:
- You may only edit arguments marked "added by you". Edits to any other argument will be ignored.
- To edit one of your arguments, refer to it by its identifier and provide a revised "argument" text and/or a revised "base_score". For example, set the base score to 0.0 if you now consider one of your arguments irrelevant.
- Every new argument must refer to the truthfulness of the claim and be a single short sentence. New pro arguments must make the claim more likely to be true; new con arguments must make it more likely to be false.
- Do not provide base scores for new arguments; they will be scored separately.
- Any "base_score" you provide for an edit must be between 0.0 and 1.0.

Return only valid JSON in exactly this structure, using empty lists for anything you do not want to change:

{
  "edits": [
    {
      "id": "pro1",
      "argument": "Revised argument text",
      "base_score": 0.50
    },
    {
      "id": "con3",
      "argument": "Revised argument text",
      "base_score": 0.50
    }
  ],
  "new_pros": [
    "New pro argument text"
  ], "new_cons": [
    "New con argument text"
  ]
}

In each edit, "argument" and "base_score" are optional; include only the fields you want to change.
\end{lstlisting}

\paragraph{Prompt 8: Posterior Stance Score (Shared Scenario)} \phantom{}

\begin{lstlisting}
System:
You are an independent debate agent. Return only valid JSON.

User:
You are one LLM debate agent.

Claim:
{claim}

Your previous stance score toward this claim was {your previous stance score} (0.00 = very likely false, 1.00 = very likely true).

The current shared argumentation framework, built collaboratively by all agents, is:
Pro arguments (supporting the truthfulness of the claim):
  - "{pro argument 1}" (base score {pro base score 1})
  ...
  - "{pro argument P}" (base score {pro base score P})
Con arguments (attacking the truthfulness of the claim):
  - "{con argument 1}" (base score {con base score 1})
  ...
  - "{con argument C}" (base score {con base score C})

Argument base scores indicate the confidence of the agent that added each argument in the validity of that argument.

Task:
Taking the shared argumentation framework into account, assign your stance score toward the truthfulness of the claim. You may keep or revise your previous view; be persuaded only to the extent the arguments warrant it.

Definitions:
- 0.0 means you think the claim is very likely false.
- 0.5 means you think the evidence is balanced or uncertain.
- 1.0 means you think the claim is very likely true.

Important constraints:
- You do not know the gold label.
- Do not generate pro/con arguments.
- Do not generate base scores.
- Do not output a support/oppose label; the script will derive it from the score.
- Output the score with exactly two decimal places.
- Return only valid JSON with exactly one key named "stance_score".

Return exactly this JSON structure:
{"stance_score": SCORE}

Replace SCORE with your own calibrated two-decimal number from 0.00 to 1.00.
\end{lstlisting}

\subsection{Multi-Turn Judge Prompts}

\paragraph{Prompt 1: LLM Judge (Private Scenario, Named Agents, Prior Stances)} \phantom{}

\begin{lstlisting}
System:
You are a careful judge for claim verification. Return only valid JSON.

User:
You are an LLM judge for a claim verification experiment.

Your task is to decide whether the claim is more likely true or false using only
the evidence below. You do not know the gold label.

Claim:
{claim}

The evidence below was produced by 3 LLM debate agents after {r} rounds of debate. In each round every agent generated its own pro and con arguments; from the second round onward each agent could revise its arguments after seeing the other agents' arguments and stances. You are shown each agent's arguments as they stood at the end of the most recent round.

Agent evidence:
Agent A1:
- prediction: {support/oppose}
- prediction score: {A1 prior stance score}
Pro arguments:
- {A1 pro argument 1}
  - base score: {A1 pro base score 1}
...
- {A1 pro argument P}
  - base score: {A1 pro base score P}
Con arguments:
- {A1 con argument 1}
  - base score: {A1 con base score 1}
...
- {A1 con argument C}
  - base score: {A1 con base score C}

Agent A2:
- prediction: {support/oppose}
- prediction score: {A2 prior stance score}
Pro arguments:
- {A2 pro argument 1}
  - base score: {A2 pro base score 1}
...
- {A2 pro argument P}
  - base score: {A2 pro base score P}
Con arguments:
- {A2 con argument 1}
  - base score: {A2 con base score 1}
...
- {A2 con argument C}
  - base score: {A2 con base score C}

Agent A3:
- prediction: {support/oppose}
- prediction score: {A3 prior stance score}
Pro arguments:
- {A3 pro argument 1}
  - base score: {A3 pro base score 1}
...
- {A3 pro argument P}
  - base score: {A3 pro base score P}
Con arguments:
- {A3 con argument 1}
  - base score: {A3 con base score 1}
...
- {A3 con argument C}
  - base score: {A3 con base score C}

Instructions:
- You are allowed to use the agents' support/oppose predictions and prediction scores as evidence. These predictions are the agents' initial assessments of the claim, formed independently before any arguments were generated, and they were not updated over the course of the debate.
- Agent identities and each agent's pro/con grouping are visible.
- Consider both the pro and con arguments and their base scores.
- The base scores are self-assessed by the agent that generated each argument.
- Assign a truth score from 0.0 to 1.0.
- A score of 0.0 means the claim is very likely false.
- A score of 0.5 means the evidence is balanced or uncertain.
- A score of 1.0 means the claim is very likely true.
- Do not output a true/false prediction; the script will convert score >= 0.5
  to true and score < 0.5 to false.
- Return only valid JSON.

Required JSON:
{
  "score": 0.0
}
\end{lstlisting}

\paragraph{Prompt 2: LLM Judge (Shared Scenario, Named Agents, Prior Stances)} \phantom{}

\begin{lstlisting}
System:
You are a careful judge for claim verification. Return only valid JSON.

User:
You are an LLM judge for a claim verification experiment.

Your task is to decide whether the claim is more likely true or false using only
the evidence below. You do not know the gold label.

Claim:
{claim}

The evidence below was produced by 3 LLM debate agents after {r} rounds of debate. The agents collaboratively built a single shared argumentation framework: in each round every agent could edit the arguments it had previously added and append new arguments of its own. You are shown the shared framework as it stood at the end of the most recent round.

Agent evidence:
The agents contributed to the shared argumentation framework in this order during this round: A1, A2, A3.

Agent predictions:
Agent A1:
- prediction: {support/oppose}
- prediction score: {A1 prior stance score}
Agent A2:
- prediction: {support/oppose}
- prediction score: {A2 prior stance score}
Agent A3:
- prediction: {support/oppose}
- prediction score: {A3 prior stance score}

Shared argumentation framework:
Pro arguments:
- {pro argument 1}
  - base score: {pro base score 1}
  - added by: Agent A1
...
- {pro argument P}
  - base score: {pro base score P}
  - added by: Agent A2
Con arguments:
- {con argument 1}
  - base score: {con base score 1}
  - added by: Agent A1
...
- {con argument C}
  - base score: {con base score C}
  - added by: Agent A3

Instructions:
- You are allowed to use the agents' support/oppose predictions and prediction scores as evidence. These predictions are the agents' initial assessments of the claim, formed independently before any arguments were generated, and they were not updated over the course of the debate.
- Agent identities and the authorship of the individual arguments in the shared framework are visible.
- Consider both the pro and con arguments and their base scores.
- The base scores are self-assessed by the agent that generated each argument.
- Assign a truth score from 0.0 to 1.0.
- A score of 0.0 means the claim is very likely false.
- A score of 0.5 means the evidence is balanced or uncertain.
- A score of 1.0 means the claim is very likely true.
- Do not output a true/false prediction; the script will convert score >= 0.5
  to true and score < 0.5 to false.
- Return only valid JSON.

Required JSON:
{
  "score": 0.0
}
\end{lstlisting}

\paragraph{Prompt 3: Judge Evidence Block (Private Scenario, Anonymised)} \phantom{}

\begin{lstlisting}
The following evidence has been anonymized. Agent identifiers are hidden,
and the order of predictions and arguments should not be interpreted as
indicating that they were generated by the same agent.

Anonymous predictions:
- prediction: {support/oppose}
- prediction score: {prediction score 1}
- prediction: {support/oppose}
- prediction score: {prediction score 2}
- prediction: {support/oppose}
- prediction score: {prediction score 3}

Pro arguments:
- {pro argument 1}
  - base score: {pro base score 1}
...
- {pro argument P}
  - base score: {pro base score P}

Con arguments:
- {con argument 1}
  - base score: {con base score 1}
...
- {con argument C}
  - base score: {con base score C}
\end{lstlisting}

\paragraph{Prompt 4: Judge Evidence Block (Shared Scenario, Anonymised)} \phantom{}

\begin{lstlisting}
The following evidence has been anonymized. Agent identifiers are hidden,
and the order of predictions and arguments should not be interpreted as
indicating that they were generated by the same agent.

Anonymous predictions:
- prediction: {support/oppose}
- prediction score: {prediction score 1}
- prediction: {support/oppose}
- prediction score: {prediction score 2}
- prediction: {support/oppose}
- prediction score: {prediction score 3}

Shared argumentation framework:
Pro arguments:
- {pro argument 1}
  - base score: {pro base score 1}
...
- {pro argument P}
  - base score: {pro base score P}
Con arguments:
- {con argument 1}
  - base score: {con base score 1}
...
- {con argument C}
  - base score: {con base score C}
\end{lstlisting}

\paragraph{Prompt 5: Judge Instruction Line Variants} \phantom{}

\begin{lstlisting}
Prediction note (first instruction bullet):

[stance = no]
- The agents' support/oppose predictions and prediction scores are intentionally hidden.

[stance = prior; named agents]
- You are allowed to use the agents' support/oppose predictions and prediction scores as evidence. These predictions are the agents' initial assessments of the claim, formed independently before any arguments were generated, and they were not updated over the course of the debate.

[stance = prior; anonymised]
- You are allowed to use the anonymous support/oppose predictions and prediction scores as evidence. These predictions are the agents' initial assessments of the claim, formed independently before any arguments were generated, and they were not updated over the course of the debate.

[stance = posterior; named agents]
- You are allowed to use the agents' support/oppose predictions and prediction scores as evidence. These predictions are the stances the agents held at the end of the most recent round, after taking the debate into account.

[stance = posterior; anonymised]
- You are allowed to use the anonymous support/oppose predictions and prediction scores as evidence. These predictions are the stances the agents held at the end of the most recent round, after taking the debate into account.

Identity note (second instruction bullet):

[scenario = private; named agents]
- Agent identities and each agent's pro/con grouping are visible.

[scenario = private; anonymised]
- Agent identities and the pairing between a prediction, pro arguments, and con arguments are hidden.

[scenario = shared; named agents]
- Agent identities and the authorship of the individual arguments in the shared framework are visible.

[scenario = shared; anonymised]
- Agent identities and the authorship of the individual arguments in the shared framework are hidden.
\end{lstlisting}

\end{document}